\documentclass{article} 
\usepackage{fullpage}

\def\fighome{.}
\def\bibhome{.}

\usepackage{url}
\usepackage{xspace}
\usepackage{hyperref}
\usepackage{times}
\usepackage{graphicx,float,pgfplots,wrapfig,sidecap,lipsum}
\usepackage{tabularx}
\usepackage{booktabs}
\usepackage{paralist}
\usepackage{algorithm}
\usepackage{algorithmicx}
\usepackage{algpseudocode}
\usepackage{amsfonts}
\usepackage{amsthm}
\usepackage{amsmath}
\usepackage{amssymb}
\usepackage{enumitem}
\usepackage{xcolor}
\usepackage[font=normal,labelfont=bf]{caption}
\usepackage{mathtools} 

\usepackage{tablefootnote}
\usepackage{subcaption}
\usepackage{psfrag}
\usepackage{tikz}
\usetikzlibrary{fit}
\usetikzlibrary{calc,shapes}
\usetikzlibrary{decorations.pathmorphing} 
\usetikzlibrary{fit}					
\usetikzlibrary{backgrounds}	

\usepackage{bm}
\usepackage{mathrsfs} 
\usepackage{stmaryrd}
\usepackage[utf8]{inputenc}
\usepackage[english]{babel}

\usepackage{multirow}
\usepackage{xspace}

\newcommand*{\mytensor}[1]{\mathcal{#1}}

\newcommand{\diag}{\mathsf{Diag}}

\newcommand\defeq{\mathrel{\overset{\makebox[0pt]{\mbox{\normalfont\tiny\sffamily def}}}{=}}}

\def\Pbb{{\mathbb P}}
\def\Rbb{\mathbb{R}}
\def\Ebb{\mathbb{E}}
\def\Est{\mathbb{{\hat{E}}}}

\def\tha{{\mbox{\tiny th}}}
\def\beq{\begin{equation}}

\def\nn{\nonumber}
\def\eeq{\end{equation}\noindent}

\newcommand{\bp}{\begin{psfrags}}
\newcommand{\ep}{\end{psfrags}}
\newcommand{\bc}{\begin{center}}
\newcommand{\ec}{\end{center}}

\newcommand\independent{\protect\mathpalette{\protect\independenT}{\perp}}
\def\independenT#1#2{\mathrel{\rlap{$#1#2$}\mkern2mu{#1#2}}}

\def\Pr{p}

\def\diag{\mathrm{diag}}

\def\cA{\mathcal{A}}

\newtheorem{theorem}{Theorem}
\newtheorem{lemma}[theorem]{Lemma}
\newtheorem{definition}[theorem]{Definition}
\newtheorem{proposition}[theorem]{Proposition}
\newtheorem{identity}[theorem]{Identity}

\newtheorem{remark}[theorem]{Remark}
\floatname{algorithm}{Procedure}

\newcommand{\pv}{\frac{c_n}{l_n} }  %

\newcommand{\wordIndTwo}{k}

\newcommand{\wordOne}{x_1}
\newcommand{\wordTwo}{x_2}
\newcommand{\wordThree}{x_3}

\newcommand{\vocabsize}{d}

\newcommand{\numTopic}{k}
\newcommand{\topicInd}{i}
\newcommand{\docInd}{n}

\newcommand{\wordInd}{j}

\newcommand{\tflda}{M_1}

\newcommand{\pmat}{ \frac{1}{2\binom{l_n}{2}} (c_n\otimes c_n - \diag(c_n))}

\newcommand{\flda}{\hat{M}_1}
\newcommand{\slda}{ \hat{M}_2 } 

\newcommand{\tlda}{ \hat{M}_3 } 

\newcommand{\nslda}{ \hat{M'}_2 } 

\newcommand{\sldak}{ \hat{M}_{2,k} } 

\newcommand{\nsldak}{ \hat{M'}_{2,k} } 

\newcommand{\ntlda}{ \hat{M'}_3 } 

\newcommand{\wt}{ \hat{M}_3(\hat{W},\hat{W},\hat{W})} 

\newcommand{\nwt}{ \hat{M'}_3(\hat{W'},\hat{W'},\hat{W'})} 

\newcommand{\swt}{ \Delta_{\wtshort}} 

\newcommand{\svbw}{ \Delta_{\bar{\mu}} } 
\newcommand{\sabw}{ \Delta_{\bar{\alpha}} } 

\newcommand{\sv}{ \Delta_{\mu} } 
\newcommand{\sa}{ \Delta_{\alpha} } 

\newcommand{\w}{ \hat{W} } 
\newcommand{\nw}{ \hat{W'} } 

\newcommand{\tslda}{ M_2 } 
\newcommand{\ttlda}{ M_3 } 

\newcommand{\sslda}{ \Delta_{2} } 
\newcommand{\stlda}{ \Delta_{3} } 

\newcommand{\ds}{ \gamma_s } 

\newcommand{\sgs}{ \sigma_k(\slda)-\sigma_{k+1}(\slda) } 

\newcommand{\norm}[1]{\left\lVert#1\right\rVert}

\newcommand{\sk}{ \sigma_k(\slda)}

\newcommand{\tdp}[1]{\tau_{\epsilon_{#1},\delta_{#1}}}

\newcommand{\sgm}[1]{ \sigma_{#1}(\slda) }

\newcommand{\ari}[1]{ \alpha^r_{#1} }

\newcommand{\mbar}[1]{ \bar{\mu}_{#1} }
\newcommand{\abar}[1]{ \bar{\alpha}_{#1} }

\newcommand{\ax}{ \frac{(\alpha_0+2)}{2\sqrt{(\alpha_0+1)\alpha_0}} }

\newcommand{\wtshort}{ \widehat{\mathcal{T}}}

\newcommand{\mpf}[1]{ \tilde{\tilde{M}}_1^{#1}}
 \newcommand{\mps}[1]{ \tilde{\tilde{M}}_2^{#1}}
 \newcommand{\mpt}[1]{ \tilde{\tilde{M}}_3^{#1}}

\newcommand{\astfootnote}[1]{
\let\oldthefootnote=\thefootnote
\setcounter{footnote}{0}
\renewcommand{\thefootnote}{\fnsymbol{footnote}}
\footnote{#1}
\let\thefootnote=\oldthefootnote
}

\newcommand{\mytitle}{An end-to-end Differentially Private Latent Dirichlet Allocation Using a Spectral Algorithm}

\title{\mytitle}
\date{}
\author{Christopher DeCarolis\\  \emph{Department of Computer Science}\\
\emph{University of Maryland}\\
 \href{mailto:cdguitar817@gmail.com}{cdguitar817@gmail.com} 
 \and Mukul Ram  \\ \emph{Department of Computer Science}\\
\emph{University of Maryland}\\
 \href{mailto:mukul.ram97@gmail.com}{mukul.ram97@gmail.com}
 \and Seyed A. Esmaeili \\ \emph{Department of Computer Science}\\
\emph{University of Maryland}\\
 \href{mailto:esmaeili@cs.umd.edu}{esmaeili@cs.umd.edu} 
\and  Yu-Xiang Wang\\ \emph{Department of Computer Science}\\
\emph{UC Santa Barbara}\\
 \href{mailto:yuxiangw@cs.ucsb.edu}{yuxiangw@cs.ucsb.edu}
   \and Furong Huang \\ \emph{Department of Computer Science}\\
\emph{University of Maryland}\\ \href{mailto:furongh@cs.umd.edu}{furongh@cs.umd.edu} }

\begin{document}
\maketitle
\begin{abstract}
We provide an end-to-end differentially private spectral algorithm for learning LDA, based on matrix/tensor decompositions, and establish theoretical guarantees on utility/consistency of the estimated model parameters. The spectral algorithm consists of multiple algorithmic steps, named as ``{edges}'', to which noise could be injected to obtain differential privacy. We identify \emph{subsets of edges}, named as ``{configurations}'', such that adding noise to all edges in such a subset guarantees differential privacy of the end-to-end spectral algorithm. We characterize the sensitivity of the edges with respect to the input and thus estimate the amount of noise to be added to each edge for any required privacy level. We then characterize the utility loss  for each configuration as a function of injected noise.  Overall, by combining the sensitivity and utility characterization, we obtain an end-to-end differentially private spectral algorithm for LDA and identify the corresponding configuration that outperforms others in any specific regime. We are the first to achieve utility guarantees under the required level of differential privacy for learning in LDA. Overall our method systematically outperforms differentially private variational inference.
\end{abstract}
\section{Introduction}
Topic modeling has been used extensively in document categorization, social sciences, machine translation and so forth.
Learning topic modeling involves projecting high dimensional observations (documents) to a lower dimensional latent structure (topics), and outputting a model parameter estimation that describes the generative process of observed documents.  
In this paper, we focus on a popular topic model, Latent Dirichlet Allocation (LDA)~\cite{blei2003latent}.
Popular methods to learning LDA, such as variational inference~\cite{blei2003latent},  
optimizes over a lower bound of the likelihood and is susceptible to local optima due to the non-convexity and high-dimensionality of the likelihood  bound. 


To provide a guaranteed consistent learning algorithm for LDA, tensor decomposition (spectral method) using method of moments is proposed~\cite{anandkumar2012spectral,anandkumar2014tensor}. 
For linearly independent topics, the tensor decomposition spectral algorithm~\cite{anandkumar2012spectral,anandkumar2014tensor} guarantees consistent recovery of the topic-word distribution (i.e. LDA model parameters), if the \emph{third order data moment tensor}, which denotes the expected co-occurrence of triplets of words in a document, is uniquely decomposed. 
A state-of-the-art tensor decomposition algorithm, called the simultaneous power method~\cite{wang2017tensor}, has been proven to recover the true components of \emph{tensors with orthogonal components}. 
Therefore, a whitening procedure that transforms the \emph{third order data moment tensor} into a \emph{tensor with orthogonal components}, which in turn can be decomposed using the simultaneous power method, is used. The whitening procedure involves matrix decomposition on the \emph{second order data moment matrix}, which denotes the expected co-occurrence of pairs of words in a document.

Overall, we have an \emph{end-to-end spectral learning algorithm} for LDA, based on matrix/tensor decomposition. This  algorithm is guaranteed to consistently recover LDA's model parameters.
We introduce an algorithmic flow graph that illustrates our  spectral learning algorithm in figure~\ref{fig:flow}, where each node corresponds to an intermediate objective required for a final output estimation and each edge $\in \{e_i\}_{i=0}^{9}$ denotes certain operation required as a step of the  spectral learning algorithm.

\begin{figure*}[!htbp]
\begin{center}
\psfrag{e0}[Bl]{$e_0$}\psfrag{e1}[Bl]{$e_1$}\psfrag{e2}[Bl]{$e_2$}\psfrag{e3}[Bl]{$e_3$}\psfrag{e4}[Bl]{$e_4$}\psfrag{e5}[Bl]{$e_5$}\psfrag{e6}[Bl]{$e_6$}\psfrag{e7}[Bl]{$e_7$}\psfrag{e8}[Bl]{$e_8$}\psfrag{e9}[Bl]{$e_9$}
\psfrag{M2}[Bl]{$\widehat{M}_2$}\psfrag{M3}[Bl]{$\widehat{M}_3$}\psfrag{W}[Bl]{$\widehat{W}$}\psfrag{TT}[Bc]{$\widehat{M}_3(\widehat{W},\widehat{W},\widehat{W})$} 
\psfrag{W+}[Bl]{$\widehat{W}^\dag$} \psfrag{vi}[Bl]{$\mbar{\topicInd}$}\psfrag{A}[Bl]{$\widehat{\mu}_\topicInd$}
\psfrag{Second Order Moments}[Bc]{Second Order Moments}
\psfrag{Third Order Moments}[Bc]{Third Order Moments}
\psfrag{Whitened Tensor}[Bc]{Whitened Tensor}
\psfrag{Eigenvectors}[Bc]{Eigenvectors}
\psfrag{Topic Word Prob}[Bl]{Topic Word Prob}
\includegraphics[width=0.9\textwidth]{\fighome/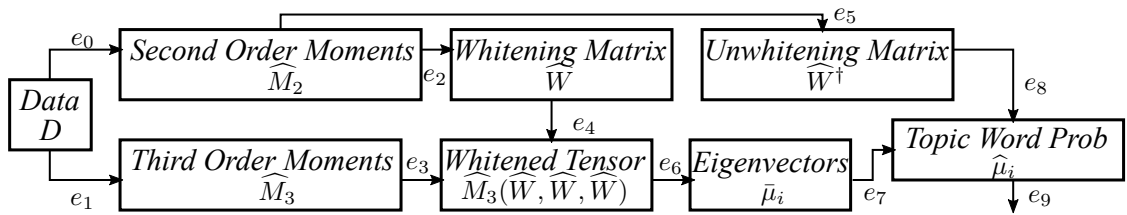}
\caption{Algorithmic flow of end-to-end spectral learning algorithm to learning  LDA topic model.}\label{fig:flow}
\end{center}
\end{figure*}

Although the spectral algorithm enjoys a provable guarantee in learning LDA, the output of this method could leak sensitive information, which limits the applicability of LDA in legal, financial and medical domains.
For instance, consider a situation in which a sensitive document corpus $D$ is kept hidden, but an adversary can obtain the output of the  spectral algorithm on $D$. 
If an additional document $d$ is added to the corpus, and the output changes for one topic $t$, an adversary can then infer that $d$ is related to $t$. 
Differential privacy (DP)~\cite{dwork2006calibrating} is a general framework for quantifying
leakage of private information from an algorithm. 
A generic method to convert an algorithm $A$ to be differentially private is {to add {\em sufficient noise\/} to $A$'s output}.  

 The \textbf{goal} of this work is: 
 \textbf{(1)} to introduce the first  differentially private algorithm that is guaranteed to recover high quality estimates (guarantee consistency) of the LDA model parameters; and 
 \textbf{(2)} to identify a mechanism that suffers least from utility loss under some level of differential privacy. 
 
 To achieve \textbf{goal (1)}, we consider injecting noise to a subset $E$ of edges $\{e_i\}_{i=0}^9$ that separates the input and the output (a cut). When $E$ is a cut, differentially privately releasing all nodes preceding the edges in $E$ guarantees the overall differential privacy according to the composition theorem and the closure to post processing. 
  For instance, adding noise to $E = (e_0,e_2)$ guarantees no privacy as the non-private information could flow to the output through the path below. However adding noise to the output $\{e_9\}$, for example, guarantees overall differential privacy as long as the global sensitivity of  $\{e_9\}$ is bounded\footnote{The global sensitivity is not bounded for $\{e_9\}$, unfortunately.}.  We call such a subset of edges as a ``configuration'' if adding noise to all edges in this configuration guarantees differential privacy of the overall algorithm. Four configurations are identified as shown in Section\ref{sec:utility-section} ``Differentially Private Spectral Algorithm''. 
 
 The amount of noise needed to achieve same level of differential privacy is different across edges as the sensitivities of edges vary, and thus utility losses caused by injecting the noise vary across edges.  Depending on which configuration we choose,  
 we may have quite different utility measures.
To obtain \textbf{goal (2)} and solve the problem of where to add the noise for least utility loss to guarantee ($\epsilon$,$\delta$)-differential privacy, we characterize the sensitivity of all the edges $\{e_i\}_{i=0}^9$ and utility bounds for all configurations with respect to the input. We then identify corresponding configuration that outperforms others in any regime. 

\paragraph{Related Works} In Appendix~\ref{sec:related_work}, we provide context on popular LDA methods, and provide a primer on differential privacy and tensor decomposition.

\paragraph{Comparison of our method against differentially private variational inference.}
The proposed approach is advantageous over differentially private VI as it (1) retains consistency guarantees, 
(2) is computationally more efficient,
(3) achieves higher accuracy in synthetic experiments, moreover, (4) does not require performing composition across multiple iterations.
However, the proposed approach (1) is less data efficient and (2) does not report posterior uncertainty in the parameter estimates.

\medskip

\noindent \textbf{Summary of Contributions}\\
\textbf{(1)} We illustrate the computation graph of our  spectral algorithm for LDA in figure~\ref{fig:flow}, based on which we list subsets of edges as configurations, on which noise could be added to guarantee differential privacy according to the composition theorem.  \\
\textbf{(2)} We bound the sensitivity of releasing various intermediate quantities on the computational graph, which leads to a number of methods for achieving both the pure-$\epsilon$-DP and the approximate $(\epsilon,\delta)$-DP.
In the cases when the global sensitivity is unbounded, we come up with a data-dependent DP that exploits the small local sensitivity.\\
\textbf{(3)} 
We achieve consistency guarantee of the algorithm by characterizing the utility loss guarantees between the \textbf{true parameters} and the differentially private parameters we obtain.
This result is stronger than traditional utility loss between the non-differentially private parameters and the differentially private parameters.\\
\textbf{(4)} Overall, 
we introduce a systematic analysis of differentially private spectral algorithm for LDA and obtain the regimes under which adding noise to each configuration guarantees the best performance.  
Using this framework, we demonstrate multiple mechanisms, which permit differentially private algorithms whose utilities are advantageous in different regimes as listed in Remark~\ref{rm:conf1_2},~\ref{rm:conf2_3} and ~\ref{rm:conf3_4}.\\
\textbf{(5)} Empirical studies confirm that our method systematically outperforms differentially private variational inference.
\section{A. Related Work}\label{sec:related_work}

Wang et al.~\cite{wang2018frequentist} establishes an impressive result of frequentist consistency and asymptotic normality of VB methods, in which the consistency is based on the assumption of achieving the \emph{optimal variational posterior}. 
However the feasibility of achieving such optimal variational posterior (which requires global optimal solution over a non-convex ELBO objective) remains unclear as the global optimality is not necessarily guaranteed.

%
This work focuses on LDA parameter estimation based on spectral
algorithms which, 
unlike EM-based algorithms\cite{park2016dp,park2016private}, guarantee parameter
recovery if a mild set of assumptions are met \cite{anandkumar2012spectral,anandkumar2014guaranteed}. 
The spectral estimation method relies on matrix decomposition and tensor decomposition
methods. 
Thus, differentially private PCA and tensor
decomposition are related to our objective. 

Differentially private PCA is an established topic, and
$(\epsilon,0)$ differentially private PCA was achieved using the
exponential mechanism in
\cite{chaudhuri2012near,kapralov2013differentially}. 
The algorithm in \cite{kapralov2013differentially} provides guarantees
but with complexity $O(d^6)$; in contrast, \cite{chaudhuri2012near}
introduces an algorithm that is near optimal but without an analysis
of convergence time. 
Although $(\epsilon,\delta)$ differential privacy is a more loose
definition of differential privacy, it leads to better utility.
Comparative experimental results show that the $(\epsilon,\delta)$ PCA
algorithm of \cite{imtiaz2016symmetric} outperform $(\epsilon,0)$
significantly, and \cite{dwork2014analyze} introduce a simple input
perturbation algorithm which achieves near optimal utility.
In our work, we follow the $(\epsilon,\delta)$ definition and use
\cite{dwork2014analyze} to obtain a differentially private matrix
decomposition when needed.

Differentially private tensor decomposition is studied in~\cite{wang2016online} with an incoherence basis assumption. It is not
clear the extent to which such an assumption holds in topic
modeling.  Only utility bounds are proved for the top
eigenvector  in ~\cite{wang2016online}. 
The authors exclude the possibility of input perturbation as that causes the privacy parameter to be lower bounded by the dimension
($\epsilon = \Omega(d)$) which is prohibitive. However, the
same analysis on the tensor of a reduced dimension would conclude that
$\epsilon = \Omega(k)$, which is acceptable for a reduced
dimension whitened tensor as $k \ll d$.  


\section{Preliminaries and Notations}
Latent Dirichlet Allocation is characterized by two model parameters: $\bm{\alpha}$, the dirichlet parameter of the topic prior, and $\bm{\mu}$, the topic word matrix.  $\bm{\alpha}$ parameterizes a dirichlet distribution, which determines the topic mixture in each document, $\bm{\mu}$ controls the word distribution per topic. 
We provide a detailed explanation of LDA in Appendix~\ref{sec:prelim_lda}. We use $d$ to denote the number of distinct words in a vocabulary, $N$ to denote the total number of documents, $k$ to denote the number of topics. The topic prior Dirichlet distribution is parameterized by $\bm{\alpha}=(\alpha_1,\ldots,
\alpha_\numTopic)$ and $\alpha_0
  =\sum_{\topicInd=1}^\numTopic \alpha_\topicInd$. For each document $n$, topic proportion is $\theta_n$, document length is $l_n$,  and word frequency vector is denoted as $c_n$. Word tokens are denoted by $x$. 
Let $D, D'$ be two datasets. We say datasets $D$ and $D'$ are adjacent (denoted by $D\sim D'$) if we can form $D'$ by replacing exactly one document from $D$.

\begin{definition}[$(\epsilon,\delta)$-\textbf{Differential Privacy}] Let $\cA:D\rightarrow Y$ be a randomized algorithm. If $\forall D\sim D', \forall S \subseteq Y$
	$\Pbb[\cA(D) \in S] \leq e^{\epsilon} \Pbb[\cA(D^\prime) \in S] + \delta $, then $\cA$ is $(\epsilon,\delta)$-differentially private.  
\end{definition}
\begin{definition}[\textbf{Local / Global Sensitivity}]
	The local sensitivity $\Delta_f(D) := \max_{D' | D'\sim D}\|f(D)-F(D')\|$ and the global sensitivity $\Delta_f := \max_{D} \Delta_f(D)$.  $\ell_p$ norm corresponds to $\ell_p$ sensitivity.
\end{definition}
In Appendix~\ref{sec:dpreview}, we review utility loss \& error, the Gaussian mechanism and the
composition theorem. 
\section{Differentially Private LDA Topic Model}
The commonly used variational inference method, which optimizes over a likelihood lower bound, provides no consistency guarantee due to the non-convexity of the likelihood function.
To achieve a guarantee on the utility for learning LDA, we use a spectral algorithm via matrix/tensor decompositions - the only existing method that provides a guarantee on the performance with enough documents. 
For the LDA model, we define the first, second, and third order LDA moments in Lemma~\ref{lm:mm2mp}. 
Using the properties of LDA, we achieve unbiased estimators of the LDA parameters by decomposing these moments into factors that correspond to each $\mu_i$, formalized in Lemma~\ref{lm:mm2mp}.  
Therefore, as long as we empirically estimate the moments $\tflda$, $\tslda$, and $\ttlda$ without bias, we obtain the model parameters $\alpha$ and $\mu$ via tensor decomposition on the empirically estimated moments. 

\begin{lemma}
[\textbf{LDA moments and Moment Decompositions Recover Model Parameters}]
\label{lm:mm2mp}
Let random variables $x_1$, $x_2$ and $x_3$ denote the first, second and third tokens in a document. Tokens are represented as one-hot encodings, i.e., $x_1=e_v$ if the first token is the v-th word in the dictionary. We define the first, second, and third order moments of LDA $\tflda$, $\tslda$ and $\ttlda$  as 
$\tflda  \defeq \Ebb[x_1]$, 
$\tslda   \defeq  \Ebb[x_1 \otimes x_2] - \frac{\alpha_0 }{\alpha_0+1} \Ebb[x_1] \otimes \Ebb[x_1]$ and 
$\ttlda  \defeq   \Ebb[x_1 \otimes x_2 \otimes x_3]  + \frac{2 \alpha_0^2}{(\alpha_0+1) (\alpha_0+2)} \Ebb[x_1] \otimes 
\Ebb[x_1] \otimes \Ebb[x_1] - \frac{1}{\alpha_0+2} \Big(\Ebb[x_1 \otimes x_2 \otimes \Ebb[x_3]] \nonumber  + \Ebb[x_1 \otimes \Ebb[x_2] \otimes x_3] + \Ebb[\Ebb[x_1] \otimes x_2 \otimes x_3]\Big)$. 
The LDA moments relate to the model parameters $\alpha$ and $\mu$ through matrix/tensor decomposition as follows
\begin{align}
\tflda& = \sum\limits_{\topicInd}^{\numTopic}
\frac{\alpha_\topicInd}{\alpha_0} \mu_\topicInd, 
\ 
\tslda  = \sum\limits_{\topicInd}^{\numTopic} 
\frac{\alpha_\topicInd}{\alpha_0(\alpha_0+1)} \mu_\topicInd\otimes  \mu_\topicInd,
\ \nonumber\\
\ttlda  &= \sum\limits_{\topicInd}^{\numTopic}
\frac{2\alpha_\topicInd}{\alpha_0(\alpha_0+1)(\alpha_0+2)} \mu_\topicInd \otimes  \mu_\topicInd\otimes  \mu_\topicInd.
%
\end{align}
\end{lemma}
Proof is in Appendix~\ref{sec:mm2mp}. Note that $\alpha_0$ is pre-specified and thus data-independent.
Using the properties of LDA, the moments are decomposed as factors shown in Lemma~\ref{lm:mm2mp}, and the factors $\mu_\topicInd$ correspond to the LDA model parameters we aim to estimate. 
According to Lemma~\ref{lm:mm2mp}, decomposing on matrix $\tslda$ only will not result in correct recovery of $\mu_i$ as there are no unique $\mu_i$'s unless $\mu_i \independent \mu_{i^\prime}$ and $\alpha_i \neq  \alpha_{i^\prime}$. The word distributions under different topics are only linearly independent instead of orthogonal. However, tensor decomposition on $\ttlda$ will yield a unique decomposition~\cite{anandkumar2014tensor}. 

\paragraph{Method of Moments \& Tensor Decomposition}
Inspired by Lemma~\ref{lm:mm2mp}, we conclude that  tensor decomposition on $\ttlda$ will result in consistent estimation of the LDA parameters $\alpha$ and $\mu_i$. We have no access to population moments $\tflda$, $\tslda$ and $\ttlda$, but do have access to word
frequency vectors $c_n$.

To solve this problem, we empirically estimate the moments $\tflda$,
$\tslda$, $\ttlda$ as in Equations~\eqref{eq:flda}\eqref{eq:slda}\eqref{eq:tlda} given the observations of word frequency vectors $c_n$, and obtain the model parameters
$\alpha$ and $\mu$ by implementing tensor decomposition on those
empirically estimated moments. 
In Lemma~\ref{m_emc} in Appendix~\ref{app:lda}, we prove that the empirical moment estimators are unbiased.

The method of moments uses the property of data moments of the LDA
model (in Lemma~\ref{lm:mm2mp}) to estimate the parameters of topic
model $\alpha$ and $\mu_\topicInd$, $\forall \topicInd\in \numTopic$.
The algorithm flow is depicted in Figure~\ref{fig:flow} and
consists of the following steps: \textbf{(1)} Using $c_n$ for document
$\forall n\in[N]$, estimate $\slda$ and $\tlda$
using equation~\eqref{eq:slda} ($e_0$ in Figure~\ref{fig:flow}) and
equation~\eqref{eq:tlda} ($e_1$ in Figure~\ref{fig:flow}). \textbf{(2)} Apply SVD on $\slda$ to obtain an estimation of the whitening matrix
$\widehat{W} \defeq \widehat{U}\widehat{\Sigma}^{-\frac{1}{2}}$, where
$\widehat{U}$ and $\widehat{\Sigma}$ are the top $\numTopic$ singular
vectors and singular values of $\slda$ ($e_2$ in
Figure~\ref{fig:flow}). \textbf{(3)} Whiten the tensor $\widehat{\mathcal{T}}
= \tlda(\widehat{W},\widehat{W},\widehat{W})$ using multilinear
operations~\footnote{The $(i,j,k)$-th  entry of the multilinear operation $\tlda(\widehat{W},\widehat{W},\widehat{W})$ is $\sum_{m,n,l}[\tlda]_{m,n,l}W_{m,i}W_{n,j}W_{l,k}$. Since $\widehat{W}$ is a $d\times k$ matrix and $\tlda$ is a $d\times d\times d$ tensor, $\tlda(\widehat{W},\widehat{W},\widehat{W})$ is a $k\times k\times k$ tensor.} on $\tlda$ with $\widehat{W}$ ($e_3$ and $e_4$ in
Figure~\ref{fig:flow}). \textbf{(4)} Implement tensor decomposition
on the whitened tensor $\widehat{\mathcal{T}}$ and denote the resulting
eigenvectors as $\mbar{\topicInd}$, $\forall
\topicInd\in[\numTopic]$ ($e_6$ in Figure~\ref{fig:flow}).
\textbf{(5)} Obtain the un-whitening matrix $\widehat{W}^\dag =
\widehat{\Sigma}^{\frac{1}{2}} \widehat{U}^\top$ ($e_5$ in
Figure~\ref{fig:flow}). \textbf{(6)} Un-whiten the singular vectors
to obtain LDA parameters: ${\widehat{\mu}_\topicInd} \propto
(\widehat{W}^\dag)^\top \mbar{\topicInd}$ and $\widehat{\alpha}_\topicInd$, $\forall
\topicInd\in\numTopic$ ($e_7$ and $e_8$ in
Figure~\ref{fig:flow}).


Our end-to-end spectral algorithm guarantees the correct learning of topic models (see
Lemma~\ref{lm:SampleCompLDA}).

\paragraph{Differentially Private LDA Problem Statement}
%
We assume that the corpus of data is held by
a trusted curator and that an analyst will query for the parameters of the
topic model. The curator has to output the model parameters $\alpha_i,
\mu_i$ in a differentially private manner with respect to the documents. While it is easy to achieve
differential privacy, the challenge is in guaranteeing high utility.
We will use the Gaussian mechanism described in Proposition~\ref{dp_prop} in this paper to achieve
$(\epsilon,\delta)$-differentially private topic modeling for each of the
configurations.  We will compute sensitivities of edges in each configuration in Section~\ref{sec:sensitivity-section} to obtain the noise level that must be added to each edge.
Our derived sensitivity and utility loss results are demonstrated in Section\ref{sec:utility-section} ``Differentially Private Spectral Algorithm''. 

\section{Sensitivity of Nodes in Algorithmic Flow}\label{sec:sensitivity-section}
We will use $\Delta_{x}$  to denote the sensitivity of $\slda$ if $x=2$, $\tlda$ if $x=3$, $\wtshort$ if $x=\wtshort$. 
A key challenge is where to add noise in the data flow shown in Figure~\ref{fig:flow}. 
First, we calculate the sensitivities at the different nodes of the data flow graph. 
Then, we consider various options and establish the utilities for different possible noise addition configurations in Section\ref{sec:utility-section} ``Differentially Private Spectral Algorithm''. 

In the following theorems, the exact forms of the sensitivity and the proofs are in appendix~\ref{app:sensitivity_proofs}. We note that the $\ell_1$ sensitivity bounds the $\ell_2$ sensitivity, similar to how $||x||_2 \leq ||x||_1$ for any given vector $x$. At times we only bound the $\ell_1$ sensitivity which in turn bounds $\ell_2$. 

\begin{theorem}[Global sensitivity of second and third order LDA moments]
\label{m23} 
Let $\sslda$ and $\stlda$ be the $\ell_1$ sensitivities for $\slda$ and $\tlda$ respectively. Both $\sslda$ and $\stlda$ are upper bounded by $\frac{2}{N}$. 
\end{theorem}

\begin{theorem}[Local sensitivity of the whitened tensor $\wtshort$] \label{swt} 
The $\ell_1$ sensitivity of the whitened tensor $\wtshort$, denoted as $\swt$, is upper bounded by $\swt = O(\frac{k^{1.5}}{N (\sigma_k(\slda))^{1.5}})$.  
\end{theorem}

\begin{theorem}[Local sensitivity of the output of tensor decomposition $\bar{\mu_i},\bar{\alpha}_i$] 
\label{svbw} 
Let $\bar{\mu}_1,\dots,\bar{\mu}_k $ and $\bar{\alpha}_1,\dots,\bar{\alpha}_k$ be the results of tensor decomposition before unwhitening.
The sensitivity of $\bar{\mu}_i$, denoted as $\svbw$, and the sensitivity of $\bar{\alpha}_i$, denoted as $\sabw$, are both upper bounded by $O(\frac{k^2}{\ds N  (\sigma_k(\slda))^{1.5}})$, where $\ds=\min_{i \in [k]} \frac{\sigma_i(\wtshort)-\sigma_{i+1}(\wtshort)}{4}$. 
\end{theorem}

\begin{theorem}[Local sensitivity of the final output  $\mu_i,\alpha_i$] \label{svaw}
The sensitivities $\sv$ and $\sa$ of the final output are upper bounded by $O(\frac{k^2 \sqrt{\sigma_1(\slda)}}{\ds N \sigma^{1.5}_k(\slda)})$. 
\end{theorem}
\noindent \textbf{Remark.} The sensitivities before the whitening are $O(\frac{1}{N})$. The whitening step  increases the sensitivity by $\frac{k^{1.5}}{\sgm{k}^{1.5}}$, leading to $O(\frac{k^{1.5}}{N (\sigma_k(\slda))^{1.5}})$. Further, the simultaneous power method for tensor decomposition increases the sensitivity by $\frac{k^{0.5}}{\ds}$, leading to $O(\frac{k^2}{\ds N (\sgm{k})^{1.5}})$. The unwhitening increases the sensitivity by $\sqrt{\sgm{1}}$, leading to $O(\frac{k^2 \sqrt{\sgm{1}}}{\ds N (\sgm{k})^{1.5}})$. 


\begin{algorithm}[!htbp]
	\caption{ $(\epsilon_1+\epsilon_1'+\epsilon, \delta_1+\delta_1'+\delta)$-Differential Privacy (\textsf{DP}) Noise Calibration}
	\label{algo:local-sensitivity}
	\begin{algorithmic}[1]
		\Require   local sensitivity of the configuration: $\mathsf{LS}$, non-\textsf{DP} output of the configuration: $f(\text{DATA})$
		\Ensure  $(\epsilon_1+\epsilon_1'+\epsilon, \delta_1+\delta_1'+\delta)$-\textsf{DP} output 
		\State  $\widehat{\sigma}_k = \sgm{k} + \textsf{Lap}(\frac{1}{N\epsilon_1})$ \Comment{$(\epsilon_1,0)$-\textsf{DP} release of $\sgm{k}$ via Laplacian mechanism}
		\State  $\tilde{\sigma}_k = \max\{0, \widehat{\sigma}_k - \frac{2}{N\epsilon_1}\log(\frac{1}{2\delta_1})\}$ 
		\Comment{high probability lower bound of $\widehat{\sigma}_k$: $\Pr(\tilde{\sigma}_k < \widehat{\sigma}_k) \ge \delta_1$}		
		\If{config \# $> 2$}
		\State  $\widehat{\ds} = \ds + \textsf{Lap}(\frac{1}{N\epsilon_1'})$ \Comment{$(\epsilon_1',0)$-\textsf{DP} release of $\ds$ via Laplacian mechanism}
		\State  $\tilde{\ds} = \max\{0, \widehat{\ds}- \frac{2}{N\epsilon_1'}\log(\frac{1}{2\delta_1'})\}$ 
		\Comment{high probability lower bound of $\widehat{\ds}$: $\Pr(\tilde{\ds} < \widehat{\ds}) \ge \delta_1'$}
		\State obtain $\tilde{\mathsf{LS}}$ by substituting $\sgm{k}$ with $\tilde{\sigma}_k$ and substituting $\ds$ with $\tilde{\ds}$ in $\mathsf{LS}$
		\Else
		\State obtain $\tilde{\mathsf{LS}}$ by replacing $\sgm{k}$ with $\tilde{\sigma}_k$ in $\mathsf{LS}$
		\State $\epsilon_1'=0$, $\delta_1'=0$
		\EndIf
		\State Return $f(\text{DATA}) + \mathcal{N}(0, {\tilde{\mathsf{LS}}^2}\tdp{})$
		\Comment{Gaussian mechanism and $\tdp{}=\frac{\sqrt{2\ln{1.25/\delta}}}{\epsilon}$}
	\end{algorithmic}
\end{algorithm} 

\subsection{Data-dependent Privacy Calibration} 
Theorem~\ref{swt}, \ref{svbw} and \ref{svaw} are local sensitivities, which are functions of the input data set.  
It is well-known that adding noise proportional to the local sensitivity does not guarantee differential privacy as the local sensitivity may be sensitive to adding or removing of individuals from the dataset and lead to the identification of individuals. 

Two seminal solutions to this problem include the smooth sensitivity framework \cite{nissim2007smooth} and the propose-test-release (PTR) framework \cite{dwork2009differential}.  The idea of the smooth sensitivity framework is to construct a smooth upper bound of the local sensitivity that is insensitive and to calibrate noise with a heavier tail that satisfies certain ``dilation'' and ``shift'' properties to achieve pure-\textsf{DP}.  The PTR framework involves proposing bounds of the local sensitivity and testing its validity. If the test is passed, we calibrate the noise according to the proposed test. PTR is often easier to use but can only provide an $(\epsilon,\delta)$-\textsf{DP} with $\delta>0$.  

In our problem, the smooth sensitivity itself is unbounded, thus we cannot apply the smooth sensitivity framework naively.  Instead, we use a variant of propose-test-release framework that releases a confidence bound of the local sensitivity in a differentially private manner, and calibrates noise accordingly, similar to the idea in~\cite{blocki2012johnson} and a more recent example in the context of data-adaptive differentially private linear regression \cite{wang2018revisiting}. We formalize the idea using the following lemma.
\begin{lemma}\label{lm:local-sensitivity}
	Let $\mathsf{LS}$ be the $\ell_p$ local sensitivity of a function $f$ on a fixed data set. Let $\tilde{\mathsf{LS}}$ obeys $(\epsilon_1,0)$-\textsf{DP} and that $\Pr[ \mathsf{LS} \geq  \tilde{\mathsf{LS}}]\leq \delta_1$ (where the probability is only over the randomness in releasing $\tilde{\mathsf{LS}}$). Then the algorithm releases $f(\text{DATA}) + Z(\epsilon,\delta, \tilde{\mathsf{LS}})$ that is $(\epsilon_1+\epsilon,\delta_1+\delta)$-\textsf{DP}, where $Z(\epsilon,\delta, \tilde{\mathsf{LS}})$ is any way of calibrating the noise for privacy (for Gaussian mechanism $Z(\epsilon,\delta, \tilde{\mathsf{LS}}) = \mathcal{N}(0,  \frac{2\tilde{\mathsf{LS}}^2\log (1.25/\delta)}{\epsilon^2})$). 	
\end{lemma}
The proof is in Appendix~\ref{sec:local-sensitivity}. 
In our problem, the local sensitivities depend on the data only through $\sgm{k}$ and $\ds$. 
A natural idea would be to privately release $\sgm{k}$ and $\ds$ and construct a high-confidence upper bound of the local sensitivity through a high-confidence lower bound of $\sgm{k}$ and $\ds$.  
We will show the global sensitivities of $\sgm{k}$ and $\sigma_i(\wtshort)$ are small, and release $\sgm{k}$ and $\sigma_i(\wtshort)$ differentially privately. 

\begin{lemma} [Global Sensitivity of $\sgm{k}$ and $\ds$] \label{lm:sensitivity-singularvalue-gap}
The sensitivities of $\sgm{k}$ and  $\ds$ are each $2/N$.
\end{lemma}
The proof is in Appendix~\ref{sec:global-sensitivity-m2}. 

\smallskip

\noindent \textbf{Calibrating Noise} Using Lemma~\ref{lm:local-sensitivity} and Lemma~\ref{lm:sensitivity-singularvalue-gap}, we describe an algorithm that guarantees ($\epsilon_1+\epsilon_1'+\epsilon,\delta_1+\delta_1'+\delta_2$)-\textsf{DP} under local sensitivity $\mathsf{LS}$ in Procedure~\ref{algo:local-sensitivity}.

\newcommand{\diffp}{differentially private\xspace}
\section{Differentially Private Spectral Algorithm}\label{sec:utility-section}
In Figure~\ref{fig:flow}, 
each node corresponds to an intermediate objective required for a final output estimation and each edge denotes certain operation required as a step of the spectral learning algorithm. 
We consider injecting noise to a subset $E$ of edges $\{e_i\}_{i=0}^9$ that separates the input and the output (a cut). When $E$ is a cut, differentially privately releasing all nodes preceding the edges in $E$ guarantees the overall differential privacy according to the composition theorem and the closure to post processing. 
We call such a subset of edges as a ``configuration'' if adding noise to all edges in this configuration guarantees differential privacy of the overall algorithm. 

In this section, we denote $\tdp{i}=\frac{\sqrt{2\ln{1.25/\delta_i}}}{\epsilon_i}$. Further, $\tilde{\sigma}_k$ and $\tilde{\ds}$ are determined by a choice of $(\epsilon_1,\delta_1)$ and $(\epsilon_1',\delta_1')$ according to Procedure~\ref{algo:local-sensitivity}.  In what follows, if noise is added to edge $e_i$, then
$\epsilon_i$ refers to the associated differential privacy parameter.

Four configurations are identified.

\subsection{Config. 1 ($e_3,e_4,e_8$): Perturbation on $\slda$ , $\tlda$ for private $\widehat{M}_3, \widehat{W}, \widehat{W}^\dag$}
In Config. 1, 
since ($e_3,e_4,e_8$) is a cut that separates the input and the output, 
we add Gaussian noise  $\mathcal{N}(0, \frac{1}{N^2}\tdp{3})$ on  $\tlda$ to ensure ($\epsilon_3,\delta_3$)-\textsf{DP} $\tlda$ for edge $e_3$,  
 noise $\mathcal{N}(0, \frac{1}{N^2}\tdp{4})$ on $\slda$ to ensure ($\epsilon_4,\delta_4$)-\textsf{DP} $W$  for edge $e_4$,
 and also noise $\mathcal{N}(0, \frac{1}{N^2}\tdp{8})$ on $\slda$ to ensure ($\epsilon_8,\delta_8$)-\textsf{DP} $W^{\dagger}$ for edge $e_8$. 

Config. 1 has a bounded global sensitivity $O(1/N)$ and allows a  pure-\textsf{DP} if we add Laplace noise using Laplacian mechanism.

\subsection{Config. 2 ($e_6,e_8$): Perturbation on $\wtshort$ and $\slda$  for private $\widehat{\mytensor{T}},\widehat{W}^\dag$}
In Config. 2, since ($e_6,e_8$) is a cut that separates the input and the output, we add Gaussian noise $\mathcal{N}(0, \frac{k^{3}}{N^2\tilde{\sigma}_k^{3}}\tdp{6})$ on $\wtshort$ to ensure  ($\epsilon_1+\epsilon_6,\delta_1+\delta_6$)-\textsf{DP} $\wtshort$ for edge $e_6$, $\mathcal{N}(0, \frac{1}{N^2}\tdp{8})$ on $\slda$ to ensure ($\epsilon_8, \delta_8$)-\textsf{DP} $W^{\dagger}$ for edge $e_8$.

In Config. 2 
the whitening matrix results from a noiseless $\slda$, but the pseudo-inverse results from a noisy $\slda$.  
We add noise to a tensor of a smaller dimension, at the expense of an increased sensitivity by a factor of $\frac{k^{3/2}}{\sigma^{3/2}_k(\slda)}$. 
To guarantee utility, we need $\epsilon_{6} = \Omega(\frac{\ds \sigma_k(\wtshort) k^{3/2}}{N \sigma^{3/2}_k(\slda)})$ and $\epsilon_{8}=\Omega(\frac{\sqrt{d}}{(\sgs) N})$. 
The dependence on $\sqrt{d}$ still remains however, as it originates from adding noise to $\slda$ which is still done for $W^{\dagger}$.

\subsection{Config. 3 ($e_7,e_8$): Perturbation on $\bar{\mu_i}$, $\bar{\alpha_i}$ and $\slda$ for private $\mbar, \widehat{W}^\dag$}
In Config. 3,  since ($e_7,e_8$) is a cut that separates the input and the output, we add Gaussian noise $\mathcal{N}(0, \frac{k^{4}}{N^2\tilde{\ds}^2\tilde{\sigma}_k^{3}}\tdp{7})$
on $\bar{\mu_i}$ and $\bar{\alpha_i}$ to ensure  ($\epsilon_1+\epsilon_1'+\epsilon_7,\delta_1+\delta_1'+\delta_7$)-\textsf{DP} $\bar{\mu_i}$ and $\bar{\alpha_i}$  for edge $e_7$, and noise $\mathcal{N}(0, \frac{1}{N^2}\tdp{8})$ 
on $\slda$ to ensure ($\epsilon_8, \delta_8$)-\textsf{DP} $W^{\dagger}$ for edge $e_8$.

This configuration adds noise to the output of the simultaneous tensor power method and thus the sensitivity after the output of the simultaneous power iteration increases by a factor of $\frac{1}{\ds}$ compared to Config. 2. 
However, according to utility loss guarantees in Theorem~\ref{thm:utility-config2} and \ref{thm:utility-config3}, the dependence on $k$ in the last term drops from $k^{2.5}$ to $k^2$ compared to Config. 2. 
This is because although the previous configuration adds noise before the decomposition at a lower sensitivity, the error in the output grows by a factor of $\frac{\sqrt{k}}{\ds}$.

\subsection{Config. 4 ($e_9$): Perturbation on output $\mu_i$, $\alpha_i$ for private $\widehat{\mu}$}
The last option we consider is to add noise to the final output. In Config. 4, since ($e_9$) is a cut that separates the input and the output, we add Gaussian noise $\mathcal{N}(0, \frac{k^{4}\sigma_1({\hat{M}_2})}{N^2\tilde{\ds}^2\tilde{\sigma}_k^{3}}\tdp{7})$
on $\mu_i$, $\alpha_i$ to ensure   ($\epsilon_1+\epsilon_1'+\epsilon_9,\delta_1+\delta_1'+\delta_9$)-\textsf{DP} $\mu_i$, $\alpha_i$  for edge $e_9$. 

This method is arguably the simplest, as the previous configurations involve the composition of multiple differentially private outputs whereas this method only adds noise to one branch. 
Adding noise to $e_9$ instead of $e_7$ means that the noise vector increases in dimension from $k$ to $d$ which makes the utility loss larger.

Though it is possible to perform input perturbation, we exclude this option because 
the $l_2$ sensitivity is {$\sqrt{2}L$} (where $L$ is the length of the longest document) which does not decay with the number of records. Therefore the utility of input perturbation is poor even with many records.

\begin{figure*}[!htbp]
\centering
\begin{subfigure}[b]{0.32\linewidth}
\centering
	\psfrag{Epsilon}[][][0.6]{Composite $\epsilon$}
	\psfrag{Error}[][][0.6]{Error}
	\psfrag{config1}[][][0.6]{config 1}
	\psfrag{config2}[][][0.6]{config 2}
	\psfrag{config3}[][][0.6]{config 2}
	\psfrag{config4}[][][0.6]{config 3}
	\psfrag{vips}[][c][0.6]{vips}
	\psfrag{vips u0}[][][0.6]{\!\!\!\!\!\!\!\!\!\! vips unnoised}
	\psfrag{0 noise}[][c][0.6]{\!\!\!\!\! unnoised}
	\psfrag{Alpha:0.1,Delta:1e-07}[][][0.6]{}
\begin{tikzpicture}[scale=0.65]

\definecolor{color0}{rgb}{0.192156862745098,0.509803921568627,0.741176470588235}
\definecolor{color1}{rgb}{0.419607843137255,0.682352941176471,0.83921568627451}
\definecolor{color2}{rgb}{0.619607843137255,0.792156862745098,0.882352941176471}
\definecolor{color3}{rgb}{0.776470588235294,0.858823529411765,0.937254901960784}
\definecolor{color4}{rgb}{0.901960784313726,0.333333333333333,0.0509803921568627}
\definecolor{color5}{rgb}{0.992156862745098,0.552941176470588,0.235294117647059}
\definecolor{color6}{rgb}{0.682352941176471,0.792156862745098,0.419607843137255}

\begin{axis}[
axis line style={white!80.0!black},
height=8cm,
legend cell align={left},
legend style={fill=none, at={(0.5,0.81)}, draw=none},
tick pos=both,
width=7cm,
x grid style={white!80.0!black},
xlabel={Composite $\epsilon$},
xmajorgrids,
xmin=0, xmax=3,
xtick style={color=white!15.0!black},
y grid style={white!80.0!black},
ylabel={Error},
ymajorgrids,
ymin=0, ymax=6,
ytick style={color=white!15.0!black}
]
\addplot [very thick, color0, mark=*, mark size=4, mark options={solid}]
table {%
0 1.70594550966198
1 1.28398947011284
2 1.18338165809891
3 1.27255437078678
};
\addlegendentry{config 1}
\addplot [very thick, color1, mark=triangle*, mark size=4, mark options={solid,rotate=180}]
table {%
0 1.33286009184361
1 1.14460181389032
2 1.34773868515684
3 1.35447131943987
};
\addlegendentry{config 2}
\addplot [very thick, color2, mark=diamond*, mark size=4, mark options={solid}]
table {%
0 1.72460899995441
1 1.72707732837965
2 1.70630116946049
3 1.71926649934667
};
\addlegendentry{config 3}
\addplot [very thick, color3, mark=square*, mark size=4, mark options={solid}]
table {%
0 1.72085812921996
1 1.72387905146362
2 1.73804689545208
3 1.7393704973639
};
\addlegendentry{config 4}
\addplot [very thick, color4, mark=triangle*, mark size=4, mark options={solid,rotate=270}]
table {%
0 5.6799999911291
1 5.51999999619819
2 5.45000000792045
3 5.38000000063364
};
\addlegendentry{vi}
\addplot [very thick, color5, mark=triangle*, mark size=4, mark options={solid,rotate=90}]
table {%
0 4.91000004403769
1 4.91000004403769
2 4.91000004403769
3 4.91000004403769
};
\addlegendentry{vi-u}
\addplot [very thick, color6, mark=asterisk, mark size=4, mark options={solid}]
table {%
0 0.216751741255057
1 0.216751741255057
2 0.216751741255057
3 0.216751741255057
};
\addlegendentry{unnoised}
\end{axis}

\end{tikzpicture}
	\caption{\scriptsize $\alpha_0=0.1$, $N=100K$}
	\label{fig:util-epsilon-small-alpha}
\end{subfigure}
\hfill
\begin{subfigure}[b]{0.32\linewidth}
	\psfrag{Epsilon}[][][0.6]{Composite $\epsilon$}
	\psfrag{Error}[][][0.6]{Error}
	\psfrag{config1}[][][0.6]{config 1}
	\psfrag{config2}[][][0.6]{config 2}
	\psfrag{config3}[][][0.6]{config 2}
	\psfrag{config4}[][][0.6]{config 3}
	\psfrag{vips}[][c][0.6]{vips}
	\psfrag{vips unnoised}[][][0.6]{\!\!\!\!\!\!\!\!\!\! vips unnoised}
	\psfrag{unnoised}[][c][0.6]{\!\!\!\!\! unnoised}
	\psfrag{Alpha:1000,Delta:1e-07}[][][0.6]{}
\begin{tikzpicture}[scale=0.65]

\definecolor{color0}{rgb}{0.192156862745098,0.509803921568627,0.741176470588235}
\definecolor{color1}{rgb}{0.419607843137255,0.682352941176471,0.83921568627451}
\definecolor{color2}{rgb}{0.619607843137255,0.792156862745098,0.882352941176471}
\definecolor{color3}{rgb}{0.776470588235294,0.858823529411765,0.937254901960784}
\definecolor{color4}{rgb}{0.682352941176471,0.792156862745098,0.419607843137255}

\begin{axis}[
axis line style={white!80.0!black},
height=8cm,
legend cell align={left},
legend style={fill=none, at={(0.5,0.4)}, draw=none},
tick pos=both,
width=7cm,
x grid style={white!80.0!black},
xlabel={Composite $\epsilon$},
xmajorgrids,
xmin=0, xmax=3,
xtick style={color=white!15.0!black},
y grid style={white!80.0!black},
ylabel={Error},
ymajorgrids,
ymin=0, ymax=2,
ytick style={color=white!15.0!black}
]
\addplot [very thick, color0, mark=*, mark size=4, mark options={solid}]
table {%
1 1.75039638685495
2 1.73373591508456
3 1.75219733663742
};
\addlegendentry{config 1}
\addplot [very thick, color1, mark=triangle*, mark size=4, mark options={solid,rotate=180}]
table {%
0 1.73710173508999
1 1.61703841625879
2 1.76184287292555
3 1.33215783924554
};
\addlegendentry{config 2}
\addplot [very thick, color2, mark=diamond*, mark size=4, mark options={solid}]
table {%
0 1.75032893555224
1 1.74025845605702
2 1.73586737625032
3 1.74163446263238
};
\addlegendentry{config 3}
\addplot [very thick, color3, mark=square*, mark size=4, mark options={solid}]
table {%
0 1.74209313149084
1 1.73758063933926
2 1.76941090909008
3 1.76075016182158
};
\addlegendentry{config 4}
\addplot [very thick, color4, mark=asterisk, mark size=4, mark options={solid}]
table {%
0 0.909406799484555
1 0.909406799484555
2 0.909406799484555
3 0.909406799484555
};
\addlegendentry{unnoised}
\end{axis}

\end{tikzpicture}
	\caption{\scriptsize $\alpha_0=1000$, $N=1000$}
	\label{fig:util-few-doc}
\end{subfigure}
\hfill
\begin{subfigure}[b]{0.32\linewidth}
	\psfrag{Epsilon}[][][0.6]{Composite $\epsilon$}
	\psfrag{Error}[][][0.6]{Error}
	\psfrag{config1}[][][0.6]{$\mathsf{config}$ 1}
	\psfrag{config2}[][][0.6]{$\mathsf{config}$ 2}
	\psfrag{config3}[][][0.6]{$\mathsf{config}$ 3}
	\psfrag{config4}[][][0.6]{$\mathsf{config}$ 4}
	\psfrag{vips}[][][0.6]{$\ \mathsf{vi}$}
	\psfrag{vips 0}[][][0.6]{$\!\!\!\!\!\mathsf{vi}$-$\mathsf{u}$}
	\psfrag{unnoised}[][][0.6]{$\!\!\!\!\!\mathsf{unnoised}$}
	\psfrag{Alpha:1000,Delta:1e-07}[][][0.6]{}
\begin{tikzpicture}[scale=0.65]

\definecolor{color0}{rgb}{0.192156862745098,0.509803921568627,0.741176470588235}
\definecolor{color1}{rgb}{0.419607843137255,0.682352941176471,0.83921568627451}
\definecolor{color2}{rgb}{0.619607843137255,0.792156862745098,0.882352941176471}
\definecolor{color3}{rgb}{0.776470588235294,0.858823529411765,0.937254901960784}
\definecolor{color4}{rgb}{0.901960784313726,0.333333333333333,0.0509803921568627}
\definecolor{color5}{rgb}{0.992156862745098,0.552941176470588,0.235294117647059}
\definecolor{color6}{rgb}{0.682352941176471,0.792156862745098,0.419607843137255}

\begin{axis}[
axis line style={white!80.0!black},
height=8cm,
legend cell align={left},
legend style={fill=none, at={(0.5,0.75)}, draw=none},
tick pos=both,
width=7cm,
x grid style={white!80.0!black},
xlabel={Composite $\epsilon$},
xmajorgrids,
xmin=0, xmax=3,
xtick style={color=white!15.0!black},
y grid style={white!80.0!black},
ylabel={Error},
ymajorgrids,
ymin=0, ymax=6,
ytick style={color=white!15.0!black}
]
\addplot [very thick, color0, mark=*, mark size=4, mark options={solid}]
table {%
0 1.69729885158889
1 0.746910098574411
2 0.70772973314965
3 0.26354904314493
};
\addlegendentry{config 1}
\addplot [very thick, color1, mark=triangle*, mark size=4, mark options={solid,rotate=180}]
table {%
0 0.615495418797017
1 0.265131358571323
2 0.262864336253306
3 0.263961730169281
};
\addlegendentry{config 2}
\addplot [very thick, color2, mark=diamond*, mark size=4, mark options={solid}]
table {%
0 1.7312842268301
1 1.71554243140836
2 1.71677117803453
3 1.7087132310229
};
\addlegendentry{config 3}
\addplot [very thick, color3, mark=square*, mark size=4, mark options={solid}]
table {%
0 1.72436587388157
1 1.6983240110188
2 1.74602937643485
3 1.72221214565222
};
\addlegendentry{config 4}
\addplot [very thick, color4, mark=triangle*, mark size=4, mark options={solid,rotate=270}]
table {%
0 5.59361999673614
1 5.46217999687047
2 5.33382999822804
3 5.23788999771289
};
\addlegendentry{vi}
\addplot [very thick, color5, mark=triangle*, mark size=4, mark options={solid,rotate=90}]
table {%
0 4.56999992174597
1 4.56999992174597
2 4.56999992174597
3 4.56999992174597
};
\addlegendentry{vi-u}
\addplot [very thick, color6, mark=asterisk, mark size=4, mark options={solid}]
table {%
0 0.263745406892583
1 0.263745406892583
2 0.263745406892583
3 0.263745406892583
};
\addlegendentry{unnoised}
\end{axis}

\end{tikzpicture}
	\caption{\scriptsize $\alpha_0=1000$, $N=100K$}
	\label{fig:util-epsilon-big-alpha}
\end{subfigure}
\caption{Error of \textbf{our method under all configurations} vs \textbf{the differentially private VI} over varying composite $\epsilon$ while fixing the composite $\delta=10^{-7}$ using $N=100k$ documents.  $\mathsf{vi}$-$\mathsf{u}$ and $\mathsf{unnoised}$ denote the non-differentially private version of variational inference and our spectral algorithm. Config. 3 overlaps with config. 4 and thus is hardly visible.}
\label{fig:util-epsilon}
\end{figure*}
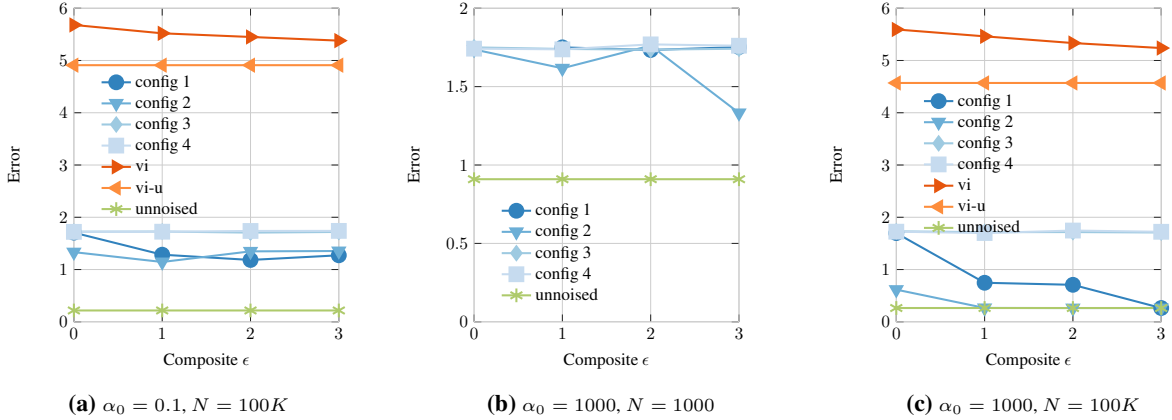

\subsection{Utility Guarantees}
For each configuration, we compute the noise
needed to obtain ($\epsilon,\delta$) differential privacy based on
sensitivity, thereby characterizing the utility with necessary noise. 
The utility of each configuration is listed in
Theorems~\ref{thm:utility-config1},~\ref{thm:utility-config2},~\ref{thm:utility-config3} and~\ref{thm:utility-config4}.
Proofs of all utility derivations are in
Appendix~\ref{app:utility}.


\begin{theorem}[Config. 1 Utility Loss]\label{thm:utility-config1}
The utility loss $\norm{\mu_i -\mu^{\mathsf{DP}}_i}$ using Config. 1  to guarantee ($\epsilon_3+\epsilon_4+\epsilon_5$, $\delta_3+\delta_4+\delta_5$)-\textsf{DP}  is 
$O(\frac{\sqrt{\sgm{1}k}}{\ds} ((\frac{\sqrt{d}}{N \sgm{k}^{3/2}}{\tdp{4}})^3+ \frac{\sqrt{d}}{N \sgm{k}^{3/2}} \tdp{3})  + \frac{\sqrt{\sgm{1} d}}{\sgm{k} N} \tdp{8}+ \sqrt{\sgm{1}+\frac{\sqrt{d}}{N} \tdp{8}} \frac{\sqrt{k}}{\ds} \Big[ (\frac{\sqrt{d}}{N \sgm{k}} \tdp{4})^3 + \frac{\sqrt{d}}{N \sgm{k}^{3/2}} \tdp{3} \Big])
.$
\end{theorem}

\begin{theorem}[Config. 2 Utility Loss]\label{thm:utility-config2}
The utility loss $\norm{\mu_i -\mu^{\textsf{DP}}_i}$ using Config. 2 to guarantee ($\epsilon_1+\epsilon_6+\epsilon_8$, $\delta_1+\delta_6+\delta_8$)-\textsf{DP}  is
$O(\frac{\sqrt{\sigma_1(\slda) k^{2.5}}}{\ds N \tilde{\sigma_k}^{3/2}} \tdp{6} + \frac{\sqrt{\sigma_1(\slda) d}}{\sk N} \tdp{8} + \sqrt{\sgm{1} + \frac{\sqrt{d}}{N} \tdp{8}} \frac{k^{2.5} \tdp{6}}{\ds N \tilde{\sigma_k}^{3/2}}).$
\end{theorem}



\begin{theorem}[Config. 3 Utility Loss]\label{thm:utility-config3}
The utility loss $\norm{\mu_i -\mu^{\mathsf{DP}}_i}$ using Config. 3  to guarantee ($\epsilon_1+\epsilon_1'+\epsilon_7+\epsilon_8$, $\delta_1+\delta_1'+\delta_7+\delta_8$)-\textsf{DP}  is $O(\frac{\sqrt{\sigma_1(\slda) k^{2.5}}}{\tilde{\ds} N \tilde{\sigma}_k^{3/2}} \tdp{7} + \frac{\sqrt{\sigma_1(\slda) d}}{\sk N} \tdp{8} + \sqrt{\sgm{1} + \frac{\sqrt{d}}{N} \tdp{8}} \frac{k^{2} \tdp{7}}{\tilde{\ds} N \tilde{\sigma}_k^{3/2}}).$
\end{theorem}
%


\begin{theorem}[Config. 4 Utility Loss]\label{thm:utility-config4}
The utility loss $\norm{\mu_i -\mu^{\mathsf{DP}}_i}$ using Config. 4 to guarantee ($\epsilon_1+\epsilon_1'+\epsilon_9$, $\delta_1+\delta_1'+\delta_9$) is
$O(\frac{\sqrt{\sgm{1}d}k^2}{\tilde{\ds} N \tilde{\sigma}_k^{3/2}} \tdp{9}).$
\end{theorem}

%

\subsection{Comparison of Configurations} \label{comparison}
The utility loss is non-vacuous if the number of data points is  $\ge \sqrt{d}$, as it depends on $O( \sqrt{d}/N )$. This sub-linear sample complexity dependence on dimension outperforms other methods. 
We present a pairwise comparison between the utilities of
different configurations. 

\begin{remark}\label{rm:conf1_2}
\textbf{Configuration 1 vs. 2:}
The utility loss in config. 1 is high compared to config. 2 as the the singular values of $\slda$ are on the order of $\frac{1}{d}$. 
%
%
Config. 1 has a $\sqrt{d}$ factor higher utility
loss compared to config. 2. 
Therefore, 
config. 2 is preferred over config. 1 in
practice.
\end{remark}

\begin{remark}\label{rm:conf2_3}
\textbf{Configuration 2 vs.  3:}
The utility loss for config. 3 is lower than that of config. 2 by a factor of $k^{0.5}$ in the last term of the utility losses, assuming the same level of differential privacy. 
However, config. 3 has the extra requirement that $\stlda \leq \frac{\ds
  \sigma_k(\wtshort)}{2\sqrt{k}}$.
Therefore the utility of config. 3 outperforms that of config. 2 only if the constraint is met.  
The advantage is enlarged when $\numTopic$ is large. 
\end{remark}

\begin{remark}\label{rm:conf3_4}
\textbf{Configuration 3 vs.  4:}
The first two terms of utility loss in config. 3 are smaller than that in config. 4. 
In the regime of $N>\sqrt{d}$ , config. 3 is preferred as the third term of config. 3 is  smaller than that of config. 4. 
\end{remark}


\section{Experiments}

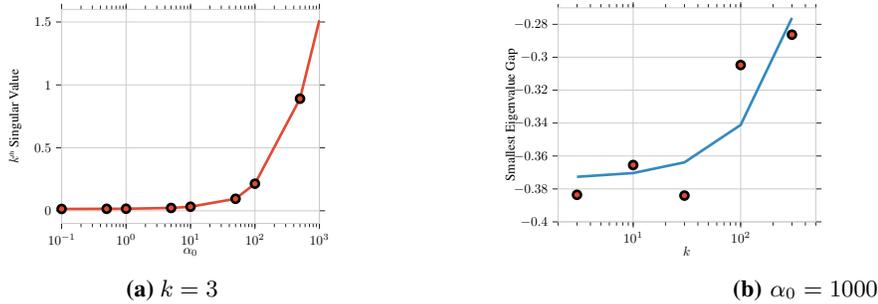
\begin{figure}[!htbp]
\centering
\begin{subfigure}[b]{0.49\linewidth}
\centering
	\psfrag{Alpha}[][][0.6]{$\alpha_0$}
	\psfrag{kth Eigenvalue}[][][0.6]{$k^{\tha}$ Singular Value}
	\psfrag{cp-seq}[][][0.6]{config 1}
	\psfrag{tt-seq}[][][0.6]{config 2}
	\psfrag{tk-seq}[][][0.6]{config 3}
	\psfrag{cp-e2e}[][][0.6]{config 4}
	\psfrag{tk-e2e}[][][0.6]{config 5}
	\psfrag{tt-e2e}[][][0.6]{config 6}
\begin{tikzpicture}[scale=0.5]

\definecolor{color0}{rgb}{0.886274509803922,0.290196078431373,0.2}

\begin{axis}[
axis line style={white!80.0!black},
log basis x={10},
tick align=outside,
tick pos=both,
x grid style={white!80.0!black},
xlabel={$\alpha_0$},
xmajorgrids,
xmin=0.1, xmax=999.999999999999,
xmode=log,
xtick style={color=white!15.0!black},
y grid style={white!80.0!black},
ylabel={$k^\tha$ Singular Value},
ymajorgrids,
ymin=-0.1, ymax=1.6,
ytick style={color=white!15.0!black}
]
\addplot [line width=2pt, color0, mark=*, mark size=3, mark options={solid,draw=black}]
table {%
0.1 0.0145533889337986
0.5 0.0156622436036462
1 0.0162148307340445
5 0.0223537530314083
10 0.032158362743745
50 0.0953294294708755
100 0.215072589496375
500 0.890369815388053
1000 1.51476698410195
};
\end{axis}

\end{tikzpicture}
	\caption{\small $k=3$} 
	\label{fig:singularvalue-alpha}
\end{subfigure}
\hfill
\begin{subfigure}[b]{0.49\linewidth}
	\psfrag{Number of Topics}[][][0.6]{$k$}
	\psfrag{Smallest Eigenvalue Gap}[][][0.6]{Smallest Singular Value Gap}
	\psfrag{cp-seq}[][][0.6]{config 1}
	\psfrag{tt-seq}[][][0.6]{config 2}
	\psfrag{tk-seq}[][][0.6]{config 3}
	\psfrag{cp-e2e}[][][0.6]{config 4}
	\psfrag{tk-e2e}[][][0.6]{config 5}
	\psfrag{tt-e2e}[][][0.6]{config 6}
\begin{tikzpicture}[scale=0.5]

\definecolor{color0}{rgb}{0.886274509803922,0.290196078431373,0.2}
\definecolor{color1}{rgb}{0.203921568627451,0.541176470588235,0.741176470588235}

\begin{axis}[
axis line style={white!80.0!black},
log basis x={10},
tick align=outside,
tick pos=both,
x grid style={white!80.0!black},
xlabel={$k$},
xmajorgrids,
xmin=2, xmax=500,
xmode=log,
xtick style={color=white!15.0!black},
y grid style={white!80.0!black},
ylabel={Smallest Eigenvalue Gap},
ymajorgrids,
ymin=-0.4, ymax=-0.27,
ytick style={color=white!15.0!black}
]
\addplot [line width=2pt, color0, mark=*, mark size=3, mark options={solid,draw=black}, only marks]
table {%
3 -0.383571196488132
10 -0.365480573990455
30 -0.38403853134846
100 -0.304713612985179
300 -0.286291961628103
};
\addplot [line width=2pt, color1]
table {%
3 -0.372635751237949
10 -0.370361031242047
30 -0.363861831253757
100 -0.34111463129474
300 -0.276122631411835
};
\end{axis}

\end{tikzpicture}
	\caption{\small \!  $\alpha_0=1000$}
	\label{fig:gap-k}
\end{subfigure}
\caption{Visualization of \textbf{(a)}  the $k^{\tha}$ singular values of $\slda$  and \textbf{(b)} the smallest singular value gap of $\wtshort$ using $100k$ documents.  }
\label{fig:singular-values}
\end{figure}

The main focus of the paper is on providing the first differentially private topic model with well understood and theoretically guaranteed utility. 
We simulate documents from an LDA model parameterized by varying choice of $\alpha$ and $\mu$ which are randomly sampled to ensure that a bursty use of a single word under certain topic is possible in our experiment.
Therefore our setting covers a wide range of hyper-parameters and captures some common irregularities in distributional properties. 
Our synthetic setting allows for direct calculation of error on parameter recovery, which is not feasible in real data. 
We compare the empirical loss of each configuration in different settings. 
In addition, we compare all configurations of our spectral algorithm against differentially private variational inference~\cite{park2016private} under the same settings.
Our algorithm universally outperforms the state-of-the-art VI quantitatively. 

\begin{figure*}[!htbp]
\begin{subfigure}[b]{0.32\linewidth}
\begin{center}
\begin{tikzpicture}[scale=0.65]

\definecolor{color0}{rgb}{0.192156862745098,0.509803921568627,0.741176470588235}
\definecolor{color1}{rgb}{0.419607843137255,0.682352941176471,0.83921568627451}
\definecolor{color2}{rgb}{0.619607843137255,0.792156862745098,0.882352941176471}
\definecolor{color3}{rgb}{0.776470588235294,0.858823529411765,0.937254901960784}
\definecolor{color4}{rgb}{0.901960784313726,0.333333333333333,0.0509803921568627}

\begin{axis}[
axis line style={white!80.0!black},
height=8cm,
legend cell align={left},
legend style={fill=none, at={(0.5,0.75)}, draw=none},
log basis x={10},
tick align=outside,
tick pos=left,
width=7cm,
x grid style={white!80.0!black},
xlabel={ \(\displaystyle \epsilon\)},
xmajorgrids,
xmin=0.000707945784384139, xmax=1.41253754462275,
xmode=log,
xtick style={color=white!15.0!black},
y grid style={white!80.0!black},
ylabel={Perplexity},
ymajorgrids,
ymin=1864.872, ymax=8549.748,
ytick style={color=white!15.0!black}
]
\addplot [very thick, color0, mark=*, mark size=4, mark options={solid}]
table {%
1 4055.62
0.1 7556.68
0.001 7588.9
};
\addlegendentry{config 1}
\addplot [very thick, color1, mark=triangle*, mark size=4, mark options={solid,rotate=180}]
table {%
1 2168.73
0.1 4153.23
0.001 4387.79
};
\addlegendentry{config 2}
\addplot [very thick, color2, mark=diamond*, mark size=4, mark options={solid}]
table {%
1 4354.71
0.1 7851.27
0.001 7896.12
};
\addlegendentry{config 3}
\addplot [very thick, color3, mark=square*, mark size=4, mark options={solid}]
table {%
1 4053.18
0.1 7646.11
0.001 7874.07
};
\addlegendentry{config 4}
\addplot [very thick, color4, mark=triangle*, mark size=4, mark options={solid,rotate=270}]
table {%
1 3755.42
0.1 8245.89
0.001 8203.59
};
\addlegendentry{vi}
\end{axis}

\end{tikzpicture}
\caption{\scriptsize{$k=25$}} 
\end{center}
\end{subfigure}
\begin{subfigure}[b]{0.32\linewidth}
\begin{center}
\begin{tikzpicture}[scale=0.65]

\definecolor{color0}{rgb}{0.192156862745098,0.509803921568627,0.741176470588235}
\definecolor{color1}{rgb}{0.419607843137255,0.682352941176471,0.83921568627451}
\definecolor{color2}{rgb}{0.619607843137255,0.792156862745098,0.882352941176471}
\definecolor{color3}{rgb}{0.776470588235294,0.858823529411765,0.937254901960784}
\definecolor{color4}{rgb}{0.901960784313726,0.333333333333333,0.0509803921568627}

\begin{axis}[
axis line style={white!80.0!black},
height=8cm,
legend cell align={left},
legend style={fill=none, at={(0.5,0.75)}, draw=none},
log basis x={10},
tick align=outside,
tick pos=left,
width=7cm,
x grid style={white!80.0!black},
xlabel={\(\displaystyle \epsilon\)},
xmajorgrids,
xmin=0.000707945784384139, xmax=1.41253754462275,
xmode=log,
xtick style={color=white!15.0!black},
y grid style={white!80.0!black},
ylabel={Perplexity},
ymajorgrids,
ymin=1033.638, ymax=4960.242,
ytick style={color=white!15.0!black}
]
\addplot [very thick, color0, mark=*, mark size=4, mark options={solid}]
table {%
1 2381.24
0.1 4367.87
0.001 4438.92
};
\addlegendentry{config 1}
\addplot [very thick, color1, mark=triangle*, mark size=4, mark options={solid,rotate=180}]
table {%
1 1212.12
0.1 2458.43
0.001 2616.71
};
\addlegendentry{config 2}
\addplot [very thick, color2, mark=diamond*, mark size=4, mark options={solid}]
table {%
1 2530.77
0.1 4571.86
0.001 4590.93
};
\addlegendentry{config 3}
\addplot [very thick, color3, mark=square*, mark size=4, mark options={solid}]
table {%
1 2271.44
0.1 4468.8
0.001 4650.89
};
\addlegendentry{config 4}
\addplot [very thick, color4, mark=triangle*, mark size=4, mark options={solid,rotate=270}]
table {%
1 2276.69
0.1 4781.76
0.001 4740.83
};
\addlegendentry{vi}
\end{axis}

\end{tikzpicture}
\caption{\scriptsize{$k=100$}}
\end{center}
\end{subfigure}
\begin{subfigure}[b]{0.32\linewidth}
\begin{center}
\begin{tikzpicture}[scale=0.65]

\definecolor{color0}{rgb}{0.192156862745098,0.509803921568627,0.741176470588235}
\definecolor{color1}{rgb}{0.419607843137255,0.682352941176471,0.83921568627451}
\definecolor{color2}{rgb}{0.619607843137255,0.792156862745098,0.882352941176471}
\definecolor{color3}{rgb}{0.776470588235294,0.858823529411765,0.937254901960784}
\definecolor{color4}{rgb}{0.901960784313726,0.333333333333333,0.0509803921568627}

\begin{axis}[
axis line style={white!80.0!black},
height=8cm,
legend cell align={left},
legend style={fill=none, at={(0.5,0.75)}, draw=none},
log basis x={10},
tick align=outside,
tick pos=left,
width=7cm,
x grid style={white!80.0!black},
xlabel={\(\displaystyle \epsilon\)},
xmajorgrids,
xmin=0.000707945784384139, xmax=1.41253754462275,
xmode=log,
xtick style={color=white!15.0!black},
y grid style={white!80.0!black},
ylabel={Perplexity},
ymajorgrids,
ymin=1078.3635, ymax=4675.2865,
ytick style={color=white!15.0!black}
]
\addplot [very thick, color0, mark=*, mark size=4, mark options={solid}]
table {%
1 2332.69
0.1 4217.5
0.001 4192.52
};
\addlegendentry{config 1}
\addplot [very thick, color1, mark=triangle*, mark size=4, mark options={solid,rotate=180}]
table {%
1 1241.86
0.1 2230.05
0.001 2492.8
};
\addlegendentry{config 2}
\addplot [very thick, color2, mark=diamond*, mark size=4, mark options={solid}]
table {%
1 1856.4
0.1 2876.61
0.001 3306.08
};
\addlegendentry{config 3}
\addplot [very thick, color3, mark=square*, mark size=4, mark options={solid}]
table {%
1 2213.43
0.1 4301.96
0.001 4320.48
};
\addlegendentry{config 4}
\addplot [very thick, color4, mark=triangle*, mark size=4, mark options={solid,rotate=270}]
table {%
1 2125.06
0.1 4466.04
0.001 4511.79
};
\addlegendentry{vi}
\end{axis}

\end{tikzpicture}
\caption{\scriptsize{$k=500$}}
\end{center}
\end{subfigure}
\caption{Perplexity scores of \textbf{our method under all configurations} vs \textbf{the differentially private VI} on Wikipedia data.}\label{fig:wiki_perp}
\end{figure*}
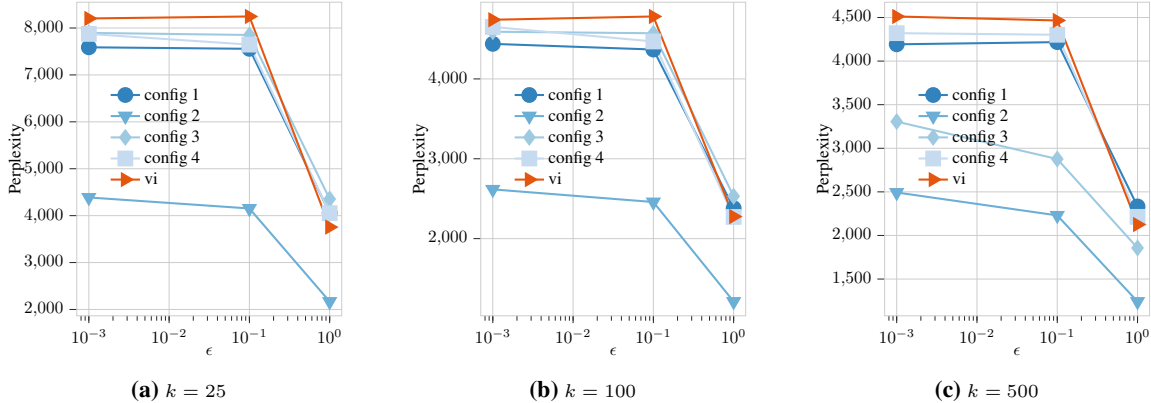

\paragraph{Evaluation Metric:} Our experiments evaluate the loss between the ground-truth $\mu$ and the estimated $\widehat{\mu}$ via a $(\epsilon,\delta)$ differentially private algorithm across varying privacy parameters (composite $\epsilon$). For each edge set and a given composite $\epsilon$, we perform a grid search over the privacy parameters for each edge in the set to select the optimal combination. When working with real data (e.g., Wikipedia), a maximum likelihood criterion could be used to select this optimal configuration. We release only differentially private likelihoods by perturbing sufficient statistics, as described in ~\cite{park2016private}. See Appendix~\ref{sec:eval_app} for a more detailed explanation of this process.

\paragraph{VI vs Spectral:} Figure~\ref{fig:util-epsilon} exhibits error for varying composite $\epsilon$ on different datasets. Under every configuration, our differentially private spectral algorithm universally outperforms differentially private variational inference, and has higher utility under the same level of privacy.

\paragraph{Small $\epsilon$ vs Large $\epsilon$:}
From figure~\ref{fig:util-epsilon-small-alpha}, config 2 performs better for low $\epsilon$, while Config. 1 performs better for larger $\epsilon$. Config. 3 and 4 perform universally worse than the other Configs as the number of topics is small. 

\paragraph{Small $\alpha_0$ vs Large $\alpha_0$:} 
In Figure~\ref{fig:util-epsilon-big-alpha}, config 1 performs the best. 
This is consistent with theoretical findings, as $\alpha_0$ is large and $N$ exceeds $d^2$. Config. 2 performs on par or worse than config. 1; this is partially due to $\slda$ having small singular values for this dataset, driving up the amount of noise added. 

\paragraph{Small Corpus vs Large Corpus:} Figure~\ref{fig:util-few-doc}, considers the limited data setting, $N=1000$. 
Config 2 emerges as better when $\alpha_0$ is large.

\paragraph{$\slda$'s singular values vs $\alpha_0$:}
The singular values of $\slda$ are positively correlated with the $\alpha_0$ parameter. Figure~\ref{fig:singularvalue-alpha} shows this correlation for large $N$. Consequently, we observe in figure~\ref{fig:util-epsilon-big-alpha}, for $N=100k$ and $\alpha_0=1000$, config. 2 has lower error.
Low $\alpha_0$ yields documents polarized to single topics, while as $\alpha_0\rightarrow \infty$, the topics mix. This yields increased variation among the singular values of $\slda$, which drives down config 2's sensitivity.

\paragraph{Singular value gap vs  number of topics:} The theoretical results posit that config 3 would likely be the best performing, but this isn't the case. This is likely due to the small singular value gap of $\gamma_s$. However, as we increase $k$ shown in Figure~\ref{fig:gap-k}, the gap increases (Figure~\ref{fig:gap-k}), which lowers the sensitivity along config. 3, leading to lower noise addition. 

\subsection{Wikipedia Dataset} 
We verified good performance of our method on the wikipedia dataset. 
Preprocessing was minimal - removing all non-alphanumeric characters and lower-casing. 
We experienced results of differing quality as we changed $\epsilon$.

As shown in Figure~\ref{fig:wiki_perp} where the quantitative results (perplexity scores) on Wikipedia are compared with variational inference, our method suffers from less utility loss under the same privacy levels. 
As we observe in the Wiki results in Figure~\ref{fig:wiki_perp}, performance of config.3 is improved under larger number of topics $k$, confirming our theory.



\section{Conclusion}

We have provided an end-to-end analysis of differentially private LDA
model using a spectral algorithm.  The algorithm involves a dataflow
that permits different locations for injecting noise.  We present a
detailed sensitivity and utility analysis for different differentially
private configurations.
%
%

We show that no configuration dominates and recommend configurations for different scenarios. Config 1 is preferable when $N\gg k$, or when the topics are highly polarized in documents (different documents do not have many topics in common). Config 2 is preferable when $N$ is small, and the topics are mixed in the dataset. Config 3 is preferable when $k$ is large. Additionally, we identified an interesting correlation between the singular values of M2 and $\alpha_0$, specifically that increasing the mixture of the topics in the documents leads to higher singular values.


 The analysis that was used can be extended to other latent
 variable models where the parameters are estimated using similar
 spectral methods such as Gaussian mixtures and Hidden Markov Models
 \cite{anandkumar2014tensor}.

\bibliographystyle{plain}\bibliography{\bibhome/supp_bib}
\newpage
\appendix
\begin{center}{\Large \textbf{Appendix: \mytitle}}\end{center}
\section{B. Differential Privacy Review}\label{sec:dpreview}

\begin{definition}[\textbf{Utility Loss \& Error}]

Let $f:D\rightarrow Y$ be a random algorithm and $f^{\mathsf{DP}}(X)$ be the differentially private version of f. For some value $x \in D$, let $y \in Y$ be the ground truth value. Then define $\norm{f(x)- f^{\mathsf{DP}}(X)}_F$ as the \textbf{utility loss} for this input. Additionally, define $\norm{y - f^{\mathsf{DP}}(X)}_F$  as the \textbf{error} for this input.
\end{definition}
The gaussian mechanism proposed in~\cite{dwork2006our} makes a random algorithm differentially private by adding specifically designed Gaussian noise to the output. 
\begin{proposition} \label{dp_prop}
[\textbf{Gaussian mechanism}]  Let $f:D \rightarrow Y$ ($Y \subset \Rbb^{k}$) be a random algorithm with $\ell_2$ sensitivity $\Delta_f$. 
Let $g \in \Rbb^{k}$ and each coordinate $g_i$ be sampled i.i.d. from $\mathcal{N}(0,\Delta^{2}_{f,\epsilon,\delta})$, where $\Delta_{f,\epsilon,\delta} =  \Delta_f  \ \tau_{\epsilon,\delta}=\frac{\Delta_f  \sqrt{2 \ln(1.25 /\delta)}}{\epsilon}$. 
Then the output $f_{\mathsf{DP}} = f + g$ is $(\epsilon,\delta)$ differentially private if $0 < \epsilon \le1$. \end{proposition}
The above bound is used for theoretical purposes only, a tighter and more general calibration of the Gaussian mechanism that does not require $\epsilon\leq 1$ was proposed in \cite{balle2018improving}.

Composition theorem~\cite{dwork2014algorithmic} provides insights on how the differential privacy is preserved under algorithm composition. 
\begin{proposition} \label{composition}
[\textbf{Composition theorem}]  Let $f^{\mathsf{DP}}_1(X), \ldots, f^{\mathsf{DP}}_n(X)$ be $n$ differentially private algorithms with privacy parameters $(\epsilon_1,\delta_1),  \ldots, (\epsilon_n,\delta_n)$. Then $g^{\mathsf{DP}}(X)=f(f^{\mathsf{DP}}_1(X), \ldots, f^{\mathsf{DP}}_n(X))$ is $(\epsilon_1+ \ldots+
\epsilon_n,\delta_1+ \ldots+\delta_n)$ differentially private.
\end{proposition}
 {This is what we called a simple composition where epsilon increases linearly.  There is an advanced composition where privacy loss for accessing for $k$ times obey that $\sqrt{k}$.}


\section{C. Latent Dirichlet Allocation}\label{sec:prelim_lda}
LDA,
despite being a bag of words model, allows modeling of the
{mixed} topics in a document to account for the more general case
in which a document belongs to several different latent classes (topics)
simultaneously. 
Latent Dirichlet Allocations has two major model parameters: topic prior $\bm{\alpha}$ and topic-word matrix $\bm{\mu}$.  Topic prior $\bm{\alpha}$ determines the topic proportions and the topic word matrix controls the word distribution per topic. 
\paragraph{Topic Proportions} 
The proportion of words in topics,  known
as \emph{topic proportion} (denoted as
$\theta_\docInd$ 
for document $n$), is drawn from a Dirichlet distribution  (topic prior) parameterized by 
$\alpha=(\alpha_1,\ldots,
\alpha_\numTopic)$, 
with density $
\small
P_{\alpha}(\theta = \theta_\docInd) =
  \frac{\Gamma(\alpha_0)}{\prod\limits_{\topicInd=1}^{\numTopic}\Gamma(\alpha_\topicInd)}
  \prod\limits_{\topicInd=1}^\numTopic
  \theta_{\docInd,\topicInd}^{\alpha_k-1}
$,  where
 $\alpha_0
  =\sum\limits_{\topicInd=1}^\numTopic \alpha_\topicInd$. 
  \paragraph{Topic-Word Matrix}
  Under a topic $i$, tokens in the documents are assumed to be generated in a conditionally independent manner through $\mu_{i}$, i.e., token $\wordOne 
 \sim \text{Cat}(\vocabsize,\mu_{i})$ where
Cat$(\vocabsize,\mu_{i})$ denotes the categorical distribution. Under different topics,
 these conditional distributions $\mu_\topicInd$ are linearly independent, $\forall
 \topicInd\in[\numTopic]$.

 With the definition of the two major parameters, we now describe the generative model of LDA topic model. The process involves generating topics first, followed by tokens.
\paragraph{Topic Generation}LDA remains simple {as each token in the corpus belongs to one of the $\numTopic$
topics only}, although tokens in the same document could belong to different topics. We denote the topic of token $j$ in document $\docInd$ as
$z_{\docInd,\wordInd}$. Therefore, topics generated are categorical $z_{\docInd, \wordInd}\in
[\numTopic]$ and distributed according to $\theta_\docInd$, i.e., $z_{\docInd, \wordInd} \sim
\text{Cat}(\numTopic,\theta_\docInd)$ where
Cat$(\numTopic,\theta_\docInd)$ denotes the categorical distribution.

\paragraph{Word Generation} Let $x$ denote the tokens. After determining the topic of the token $j$, $z_{\docInd, \wordInd}$,  token $j$ is generated conditionally independently through $\mu_{z_{\docInd, \wordInd}}$, i.e., token $
 \sim \text{Cat}(\vocabsize,\mu_{z_{\docInd, \wordInd}})$.  In a document $\docInd$, if the ${j^\prime}^{th}$  token $x_{n,j^\prime}$ is the $v$-th word in the dictionary, then  $x_{n,j^\prime} =e_v$ where $e_v$ is a one-hot encoding, i.e.,  $x_{n,j^\prime} (j) = 0$ $\forall j \neq v$ and $x_{n,j^\prime}(j) = 1$ if $j=v$. Let $l_n$ be the length of document $n$, random realizations of token $x$, i.e., $\{x_{n,j^\prime}\}_{j^\prime=1}^{l_n}$,  are i.i.d.

\paragraph{Term-Document Matrix}The term-document matrix $D \in \mathbb{N}_{0}^{d \times N}$. The $n^{th}$ column in $D$ is denoted by $c_n$, where its $j^{th}$ component 
$c_n(j)=$ number of times word $j$ in the vocabulary appeared in document $n$.
This means that $c_n = \sum_{j^\prime=1}^{l_n} x_{n,j^\prime}$ where $l_n$ is the number of words in document $n$. Clearly, $l_n = \sum_{j}^{d} c_n(j) = \norm{c_n}_1$. 
\section{D. Method of Moments for Latent Dirichlet Allocation}\label{app:lda}

\paragraph{Empirical Moment Estimators}
The moments that we obtain are not the population moments but rather empirically estimated moments from the given data set. We list the forms of first, second, and third order empirical moment estimators for the single topic case as shown in \cite{zou2013contrastive}. Given a document $n$, the following quantities are calculated. 
\begin{align}
\mpf{n} & = \pv \\
\mps{n} &= \pmat \\
\mpt{n} &=  \frac{1}{6\binom{l_n}{3}}\Big( c_{n}\otimes c_{n}\otimes c_{n}+ 2 \sum\limits_{i=1}^{d} c_{n}(i) (e_i\otimes e_i\otimes e_i) \nonumber \\
&- \sum\limits_{i=1}^{d}\sum\limits_{j=1}^{d} c_{n}(i) c_{n}(j) (e_i\otimes e_i\otimes e_j+e_i\otimes e_j\otimes e_j+e_j\otimes e_i\otimes e_j)\Big)
\end{align}
The empirically estimated moments are the averages of these quantities over the entire data set. Specifically, 
\begin{lemma} {Single Topic Empirical Moment Estimators(Propositions 3 and 4 in \cite{zou2013contrastive})} \label{single_topic_est}
\begin{align*}
& \Est[x_1] = \frac{1}{N} \sum_{n=1}^{N} \mpf{n} \\
& \Est[x_1 \otimes x_2] = \frac{1}{N} \sum_{n=1}^{N} \mps{n} \\
& \Est[x_1 \otimes x_2 \otimes x_3] = \frac{1}{N} \sum_{n=1}^{N} \mpt{n} \\
\end{align*}
Further these moments are unbiased, i.e.:
\begin{align*}
& \Ebb[\Est[x_1]] = \Ebb[\frac{1}{N} \sum_{n=1}^{N} \mpf{n}] = \Ebb[x_1] \\
& \Ebb[\Est[x_1 \otimes x_2]] = \Ebb[\frac{1}{N} \sum_{n=1}^{N} \mps{n}] = \Ebb[x_1 \otimes x_2] \\
& \Ebb[\Est[x_1 \otimes x_2 \otimes x_3]] = \Ebb[\frac{1}{N} \sum_{n=1}^{N} \mpt{n}] = \Ebb[x_1 \otimes x_2 \otimes x_3] \\
\end{align*}
\end{lemma}
Note that this lemma implies that: $\Ebb[\mpf{n}] = \Ebb[x_1],  \Ebb[\mps{n}] = \Ebb[x_1\otimes x_2], $ and that $\Ebb[\mpt{n}] = \Ebb[x_1\otimes x_2 \otimes x_3]$ for any sampled document $n$.

We extend the single topic moment estimators of \cite{zou2013contrastive} to the LDA case.
\begin{lemma}{Empirical Moment estimators for LDA}\label{df:momentEst}
{\small{
\begin{align}
\flda & =\frac{1}{N} \sum_{n=1}^{N} \mpf{n} \label{eq:flda}\\
\slda &= \frac{1}{N} \sum_{n=1}^{N} \Bigg[ \mps{n} \Bigg] 
-  \frac{a}{2\binom{N}{2}}\Bigg[\sum_{m,n=1}^{N} \mpf{n} \otimes \mpf{m} - \sum_{n=1}^{N} \mpf{n} \otimes \mpf{n} \Bigg] 
\label{eq:slda}
\\
\tlda & = d \Bigg[ \frac{1}{N}   \sum_{n=1}^{N} \tilde{\tilde{M}}_3^n 
+ \mathbf{B}_1 + \mathbf{B}_2 + \mathbf{B}_3 + \mathbf{b} \label{eq:tlda} \Bigg]
\end{align}
}}
where
\begin{align}
\mathbf{B}_1 &\defeq \frac{b}{2\binom{N}{2}}\Bigg[ \Big(\sum\limits_{n=1}^{N} \tilde{\tilde{M}}_2^n\Big) \otimes \Big(\sum\limits_{n=1}^{N} \tilde{\tilde{M}}_1^n\Big)
\Bigg],\\
\mathbf{b} &\defeq  c \Bigg[  \Big(\sum\limits_{n=1}^{N} \tilde{\tilde{M}}_1^n\Big) \otimes \Big(\sum\limits_{n=1}^{N} \tilde{\tilde{M}}_1^n\Big)  \otimes \Big(\sum\limits_{n=1}^{N} \tilde{\tilde{M}}_1^n\Big)  
\Bigg],
\end{align}
$\mathbf{B}_2$ and $\mathbf{B}_3$ are formed from $\mathbf{B}_1$ by permuting, i.e., $[\mathbf{B}_2]_{ijk}=[\mathbf{B}_1]_{ikj}$ and $[\mathbf{B}_3]_{ijk}=[\mathbf{B}_1]_{kij}$. Further, $a=\frac{\alpha_0}{\alpha_0+1},b=\frac{-\alpha_0}{\alpha_0+2},$ $c =  \frac{2 \alpha_0^2}{(\alpha_0+1)(\alpha_0+2)}$, and $d = \frac{\alpha_0(\alpha_0+1)(\alpha_0+2)}{2} $. 
\end{lemma}
Now we prove that these estimators are unbiased. 
\begin{lemma}[The LDA Moment Estimators are Unbiased]\label{m_emc}
The estimators defined in definition~\ref{df:momentEst} are unbiased, i.e., 
\begin{align}
\Ebb[\flda] & = \tflda \\
\Ebb[\slda] & = \tslda \\
\Ebb[\tlda] & = \ttlda
\end{align}
\end{lemma}
\begin{proof}
\textbf{First order moment}:
\begin{align*} 
\Ebb [\flda] &= \Ebb[\frac{1}{N} \sum_{n=1}^{N} \mpf{n}] = \frac{1}{N}  \sum_{n=1}^{N} \Ebb[\mpf{n}] = \frac{1}{N}  \sum_{n=1}^{N} \Ebb[\pv] \\
&= \frac{1}{N}  \sum_{n=1}^{N}  \frac{1}{l_n} \Ebb[c_n] =  \frac{1}{N}  \sum_{n=1}^{N}  \frac{1}{l_n} \Ebb[ \sum_{i=1}^{l_n} x_{n,i}] = \frac{1}{N}  \sum_{n=1}^{N}  \frac{1}{l_n} \sum_{i=1}^{l_n} \Ebb[x_{n,i}] \\
&= \frac{1}{N}  \sum_{n=1}^{N}  \frac{1}{l_n} \sum_{i=1}^{l_n} \Ebb[x_1]  = \frac{1}{N}  \sum_{n=1}^{N}  \frac{1}{l_n} l_n \Ebb[x_1] = \frac{1}{N}  N \Ebb[x_1]  = \Ebb[x_1]  = \tflda
\end{align*}
 \textbf{Second order moment}:
The first term of $\slda$ is actually the estimator the single-topic second order moment and  $\Ebb[\frac{1}{N} \sum_{n=1}^{N} \mps{n}] ] = \Ebb[x_1 \otimes x_2]$ see proposition 3 in \cite{zou2013contrastive} and its appendix for the proof. Now we have: 
\begin{align*}
\Ebb\Bigg[ \frac{a}{2\binom{N}{2}}\Bigg[\sum_{m,n=1}^{N} \mpf{n} \otimes \mpf{m} - \sum_{n=1}^{N} \mpf{n} \otimes \mpf{n} \Bigg] \Bigg] = & \Ebb\Bigg[ \frac{a}{2\binom{N}{2}}\Bigg[ \sum_{\substack{m=1\\n=1\\ m\neq n }}^{N} \mpf{n} \otimes \mpf{m} + \sum_{n=1}^{N} \mpf{n} \otimes \mpf{n} - \sum_{n=1}^{N} \mpf{n} \otimes \mpf{n} \Bigg] \Bigg] \\ 
=&\Ebb\Bigg[ \frac{a}{2\binom{N}{2}} \sum_{\substack{m=1\\n=1\\ m\neq n }}^{N} \mpf{n} \otimes \mpf{m} \Bigg] \\
=& \frac{a}{2\binom{N}{2}}  \sum_{\substack{m=1\\n=1\\ m\neq n }}^{N} \Ebb[\mpf{n}] \otimes \Ebb[\mpf{n}] \\
=& \frac{a}{2\binom{N}{2}}  \sum_{\substack{m=1\\n=1\\ m\neq n }}^{N} \Ebb[x_1] \otimes  \Ebb[x_1] \\
=& \frac{a}{N(N-1)} N(N-1) \Ebb[x_1] \otimes  \Ebb[x_1] \\
=& a_2 \Ebb[x_1] \otimes  \Ebb[x_1]
\end{align*}
Thus, we have that:
\begin{align*}
\Ebb[\slda] = \Ebb[x_1\otimes x_2] - \frac{\alpha_0}{\alpha_0+2}\Ebb[x_1] \otimes \Ebb[x_1] =M_2 
\end{align*}
\textbf{Third order moment}:
Similar to the second order moment, the first term of $\tlda$ is the estimator the single-topic second order moment and  $\Ebb[\frac{1}{N} \sum_{n=1}^{N} \mpt{n}] ] = \Ebb[x_1 \otimes x_2 \otimes x_3]$ as shown in proposition 4 in \cite{zou2013contrastive} and proved in its appendix. We need to prove that (1): $\Ebb[\mathbf{B}_1] = b \Ebb[x_1 \otimes x_2 \otimes \Ebb[x_3] ]$, note that $\Ebb[x_3] = M_1$ and (2): $\Ebb[\mathbf{b}] = c E[x_1] \otimes E[x_1] \otimes E[x_1] = c M_1 \otimes M_1 \otimes M_1 \otimes M_1$. Since $\mathbf{B}_2$ and $\mathbf{B}_3$ are permuted version of $\mathbf{B}_1$ their proofs follow from the proof of $\mathbf{B}_1$. 

For $\mathbf{B}_1$ we simplify the expression and then show that the expectation of the resultant is equal to the desired moment:
\begin{align*}
\Ebb[\mathbf{B}_1 ] = & \frac{b}{2\binom{N}{2}} \Ebb\Bigg[ \Big(\sum\limits_{n=1}^{N} \tilde{\tilde{M}}_2^n\Big) \otimes \Big(\sum\limits_{n=1}^{N} \tilde{\tilde{M}}_1^n\Big) - 
 \sum\limits_{n=1}^{N} \Big(\tilde{\tilde{M}}_2^n \otimes \tilde{\tilde{M}}_1^n\Big) \Bigg] \\ 
= & \frac{b}{2\binom{N}{2}} \Ebb\Bigg[ \sum_{\substack{m=1,n=1\\ m\neq n }}^{N} \Big(\tilde{\tilde{M}}_2^n \otimes  \tilde{\tilde{M}}_1^m\Big) + \sum\limits_{n=1}^{N} \Big(\tilde{\tilde{M}}_2^n \otimes \tilde{\tilde{M}}_1^n\Big) - 
 \sum\limits_{n=1}^{N} \Big(\tilde{\tilde{M}}_2^n \otimes \tilde{\tilde{M}}_1^n\Big) \Bigg]\\
= & \frac{b}{2\binom{N}{2}} \Ebb\Bigg[ \sum_{\substack{m=1,n=1\\ m\neq n }}^{N} \Big(\tilde{\tilde{M}}_2^n \otimes  \tilde{\tilde{M}}_1^m\Big) \Bigg]\\
= & \frac{b}{2\binom{N}{2}} \sum_{\substack{m=1,n=1\\ m\neq n }}^{N}  \Ebb\Big[\tilde{\tilde{M}}_2^n\Big]  \otimes \Ebb\Big[\tilde{\tilde{M}}_1^m\Big] \\
= & \frac{b}{2\binom{N}{2}} \sum_{\substack{m=1,n=1\\ m\neq n }}^{N} \Ebb[x_1\otimes x_2]  \otimes \Ebb[x_3] \\
= & \frac{b}{N(N-1)} N(N-1) \Ebb[x_1\otimes x_2] \otimes \Ebb[x_3]  \\
= & b  \Ebb[x_1\otimes x_2] \otimes \Ebb[x_3] \\
= & b  \Ebb\big[x_1 \otimes x_2 \otimes \Ebb[x_3]\big]
\end{align*}
For $\mathbf{b}$ identity \ref{I2} is applied, this leads to the following 
\begin{align*}
\Ebb[\mathbf{b}] = & \frac{c}{6\binom{N}{3}} \Ebb\Bigg[ \Bigg( \sum_{i=1}^N (\mpf{n})^{\otimes^{3}} + 3\sum_{\substack{n=1,m=1 \\ n\neq m}}^{N,N} (\mpf{n})^{\otimes^2} \mpf{m}  +  \sum_{\substack{n=1,m=1,p=1 \\ n \neq m,m \neq p,p \neq n}}^{N,N,N} \mpf{n} \otimes \mpf{m} \otimes \mpf{p} \\ &- 3 \sum\limits_{m=1}^{N}\Big(\sum\limits_{n=1}^{N} \Big(\tilde{\tilde{M}}_1^n\Big)^{\otimes^2}\otimes \Big(\tilde{\tilde{M}}_1^m\Big)\Big)  + 2 \sum\limits_{n=1}^{N} \Big(\tilde{\tilde{M}}_1^n\Big)^{\otimes^{3}} \Bigg)\Bigg] \\ =& \frac{c}{6\binom{N}{3}} \Ebb\Bigg[ \sum_{\substack{n=1,m=1,p=1 \\ n \neq m,m \neq p,p \neq n}}^{N,N,N} \mpf{n} \otimes \mpf{m} \otimes \mpf{p}\Bigg] \\= &\frac{c}{N(N-1)(N-2)} (N)(N-1)(N-2)\Ebb[\mpf{n}] \otimes \Ebb[\mpf{m}] \otimes \Ebb[\mpf{p}] \\ =& c\Ebb[x_1] \otimes \Ebb[x_1] \otimes \Ebb[x_1] 
\end{align*}
Combing these results and plugging the values for $a,b,$and $c$ we get:
\begin{align*}
\Ebb[\tlda] = & \Ebb[x_1\otimes x_2 \otimes x_3] - \frac{\alpha_0}{\alpha_0+2}\Bigg( \Ebb[x_1\otimes x_2 \otimes \Ebb[x_3]] + \Ebb[x_1\otimes \Ebb[x_2] \otimes x_3] + \Ebb[\Ebb[x_1] \otimes x_2 \otimes x_3] \Bigg) \\ + & \frac{2\alpha_0^2}{(\alpha_0+1)(\alpha_0+2)}  \Ebb[x_1] \otimes \Ebb[x_1] \otimes \Ebb[x_1]=M_3
\end{align*}
\end{proof}

\section{E. Lemmas regarding Dirichlet Moments}\label{sec:dirlemm}
This section introduces two lemmas regarding the moments of the dirichlet distribution that will be useful for the proof of~\ref{lm:mm2mp}.

\subsection{Dirichlet Moments}
\begin{lemma}\label{lm:dirichlet}The first, second and third moments of dirichlet distribution are
\begin{align}
\mathbb{E}[\theta] &= \frac{1}{\alpha_0} \alpha \\
\mathbb{E}[\theta\otimes\theta] &= \frac{1}{\alpha_0(\alpha_0+1)}[\alpha\otimes \alpha + \sum\limits_{t=1}^T \alpha_t e_t \otimes e_t]\\
\mathbb{E}[\theta\otimes\theta\otimes\theta] &= \frac{1}{\alpha_0(\alpha_0+1)(\alpha_0+2)}[\alpha\otimes\alpha\otimes\alpha 
+ \sum\limits_{t=1}^T \alpha_t e_t \otimes e_t\otimes \alpha \nonumber\\
+& \sum\limits_{t=1}^T \alpha_t \alpha \otimes e_t\otimes e_t 
+ \sum\limits_{t=1}^T \alpha_t e_t \otimes \alpha\otimes e_t
+ 2  \sum\limits_{t=1}^T \alpha_t e_t \otimes e_t\otimes e_t]
\end{align}
\end{lemma}

\subsection{Raw Moments}
\begin{lemma}\label{lm:rawMoments}
\begin{align}
\mathbb{E}[\wordOne] & = \mu \mathbb{E}[\theta] \\
\mathbb{E}[\wordOne\otimes \wordTwo]  &= \mu \mathbb{E}[\theta\otimes \theta]\mu^\top \\
\mathbb{E}[\wordOne\otimes \wordTwo\otimes \wordThree]  &= \mathbb{E}[\theta\otimes \theta\otimes \theta](\mu,\mu,\mu)
\end{align}
\end{lemma}

\begin{proof}
\textbf{First Order Moments} \label{proof:first}
Let us omit $\docInd$ and use $\wordOne$ to denote a token in any document, and we will use $\wordTwo$ and $\wordThree$ to denote other two tokens in the same document.
The the expectation of a token is
\begin{equation}
\mathbb{E}[\wordOne] = \mathbb{E}[\wordTwo] = \mathbb{E}[\wordThree] = \mathbb{E}[\mathbb{E}[\wordOne|\theta]] = \mu \mathbb{E}[\theta]
\end{equation}
This is called the first order moment.

\textbf{Second Order Moments}\label{proof:second}
The second order moment is defined as 
\begin{align}
\mathbb{E}[\wordOne\otimes \wordTwo]  & = \mathbb{E}[\mathbb{E}[\wordOne\otimes \wordTwo|\theta]] \\
& = \sum_{i,i^\prime} \mathbb{E}[\wordOne\otimes \wordTwo | z_{\docInd, \wordInd}=e_i, z_{\docInd, \wordIndTwo}=e_{i^\prime}]P(z_{\docInd, \wordInd}=e_i, z_{\docInd, \wordIndTwo}=e_{i^\prime})\\
& = \sum_{i,i^\prime} \mathbb{E}[\wordOne| z_{\docInd, \wordInd}=e_i]\otimes  \mathbb{E}[\wordTwo|z_{\docInd, \wordIndTwo}=e_{i^\prime}]P(z_{\docInd, \wordInd}=e_i, z_{\docInd, \wordIndTwo}=e_{i^\prime})\\
& = \sum_{i,i^\prime} \mu e_i \otimes (\mu e_{i^\prime}) P(z_{\docInd, \wordInd}=e_i, z_{\docInd, \wordIndTwo}=e_{i^\prime})\\
& = \mu  \sum_{i,i^\prime} e_i \otimes e_{i^\prime} P(z_{\docInd, \wordInd}=e_i, z_{\docInd, \wordIndTwo}=e_{i^\prime}) \mu^\top\\
& = \mu \mathbb{E}[\theta\otimes \theta]\mu^\top 
\end{align}

\textbf{Third Order Moments}\label{proof:third}
The third order moment is defined as 
\begin{equation}
\mathbb{E}[\wordOne\otimes \wordTwo\otimes \wordThree]  = \mathbb{E}[\mathbb{E}[\wordOne\otimes \wordTwo\otimes \wordThree|\theta]] = \mathbb{E}[\theta\otimes \theta\otimes \theta](\mu,\mu,\mu)
\end{equation}
To clarify the notations, $x \otimes y$ is a $length(x)$-by-$length(y)$ matrix which has entries $[x \otimes y]_{i,j} = x_i y_j$. And $ \mathbb{E}[\theta\otimes \theta\otimes \theta](\mu,\mu,\mu)$ is a tucker with core tensor $ \mathbb{E}[\theta\otimes \theta\otimes \theta]$ and projection $\mu$ in all three modes.
\end{proof}
\section{F. Proof of Lemma~\ref{lm:mm2mp}}\label{sec:mm2mp} The lemma relates the LDA moments to the model parameters $\alpha$ and $\mu$.
\begin{proof}

In order to prove this relation, we combine Lemmas~\ref{lm:dirichlet} and Lemma~\ref{lm:rawMoments} to prove the forms of $\tflda$, $\tslda$ and $\ttlda$ in Lemma~\ref{lm:mm2mp} as follows. 
\begin{align}
\tflda= \mathbb{E}[\wordOne]  &= \mu \mathbb{E}[\theta] \\
& =\sum\limits_{\topicInd=1}^\numTopic \frac{\alpha_\topicInd}{\alpha_0} \mu_\topicInd
\end{align}

\begin{align}
\tslda & =  \mathbb{E}[\wordOne \otimes \wordTwo] - \frac{\alpha_0}{\alpha_0+1 } \mathbb{E}[\wordOne] \otimes \mathbb{E}[\wordOne]\\
&=\mathbb{E}[\theta\otimes \theta](\mu,\mu)- \frac{1}{\alpha_0(\alpha_0+1)}M_1\otimes M_1 \\
& = \sum\limits_{\topicInd=1}^\numTopic \frac{\alpha_\topicInd}{\alpha_0(\alpha_0+1)} \mu_\topicInd \otimes \mu_\topicInd
\end{align}

\begin{align}
\ttlda  &=    \Ebb[x_1 \otimes x_2 \otimes x_3] - \frac{1}{\alpha_0+2} (\Ebb[x_1 \otimes x_2 \otimes \Ebb[x_3]] + \Ebb[x_1 \otimes \Ebb[x_2] \otimes x_3] \nonumber\\
&+ \Ebb[\Ebb[x_1] \otimes x_2 \otimes x_3]) + \frac{2}{\alpha_0(\alpha_0+1) (\alpha_0+2)} \Ebb[x_1] \otimes \Ebb[x_1] \otimes \Ebb[x_1]\\
& = \mathbb{E}[\theta\otimes \theta \otimes \theta](\mu,\mu,\mu)\\
& - \frac{1}{\alpha_0+2}\{\mathbb{E}[\theta\otimes \theta \otimes \mathbb{E}[\theta]] - \mathbb{E}[\theta\otimes \mathbb{E}[\theta] \otimes \theta] -\mathbb{E}[\mathbb{E}[\theta] \otimes \theta \otimes \theta]\}(\mu,\mu,\mu)\\
& +\frac{2}{\alpha_0(\alpha_0+1) (\alpha_0+2)} M_1\otimes M_1\otimes M_1\\
 & = \sum\limits_{\topicInd}^{\numTopic}
\frac{2\alpha_\topicInd}{\alpha_0(\alpha_0+1)(\alpha_0+2)} \mu_\topicInd \otimes  \mu_\topicInd\otimes  \mu_\topicInd
\end{align}
\end{proof}
%


\section{G. Computing $\mathbb{E}[\wordOne], \mathbb{E}[\wordOne\otimes \wordTwo],\mathbb{E}[\wordOne\otimes \wordTwo\otimes \wordThree]$}\label{sec:expcomp}
Let $c_n$ be the count vector. 
\begin{align}
\mathbb{E}[\wordOne] & = \frac{1}{N} \sum\limits_{n=1}^N \frac{1}{l_n} c_n\\
\mathbb{E}[\wordOne\otimes \wordTwo] & = \frac{1}{N} \sum\limits_{n=1}^N \frac{1}{2{\binom{l_n}{2}}} [c_n\otimes c_n-\diag(c_n)]\\
\mathbb{E}[\wordOne\otimes \wordTwo\otimes \wordThree] & = \frac{1}{N} \sum\limits_{n=1}^N \frac{1}{6{\binom{l_n}{3}}} \left[c_n\otimes c_n \otimes c_n \right.\nonumber\\
& - \sum\limits_{i,j=1}^{D} c_n(i) c_n(j) e_i\otimes e_i \otimes e_j - \sum\limits_{i,j=1}^{D} c_n(i) c_n(j) e_i\otimes e_j \otimes e_i \nonumber\\
& -\sum\limits_{i,j=1}^{D} c_n(i) c_n(j) e_j\otimes e_i \otimes e_i \nonumber\\
& \left.+2 \sum\limits_{i,j=1}^{D} c_n(i)  e_i\otimes e_i \otimes e_i\right]
\end{align}

\section{H. Dirichlet Moments}\label{sec:dirmom}
We characterize the core tensor $\mathbb{E}[\theta\otimes \theta\otimes \theta]$, where $i,j,k$-th entry of the tensor is $\mathbb{E}[\theta_i \theta_j \theta_k]$.
Now the moments of topic models are reduced to moments of topic proportions. Since topic proportions are dirichlet distributed, we characterize the dirichlet moments.

\paragraph{univariate moments for $i$-th coordinate of dirichlet variable $\theta_i$} We know that $\mathbb{E}[\theta_i^p] = \frac{\Gamma(\alpha_i+p)}{\Gamma(\alpha_i)} \frac{\Gamma(\alpha_0)}{\Gamma(\alpha_0+p)}$, therefore
\begin{align}
\mathbb{E}[\theta_i] & = \frac{\alpha_i}{\alpha_0} \\
\mathbb{E}[\theta_i^2] & = \frac{\alpha_i(\alpha_i+1)}{\alpha_0(\alpha_0+1)}\\
\mathbb{E}[\theta_i^3] & = \frac{\alpha_i(\alpha_i+1)(\alpha_i+2)}{\alpha_0(\alpha_0+1)(\alpha_0+2)}
\end{align}
\paragraph{bivariate moments} We know that $\mathbb{E}[\theta_i^p\theta_j^q] = \frac{\Gamma(\alpha_i+p)}{\Gamma(\alpha_i)}\frac{\Gamma(\alpha_j+q)}{\Gamma(\alpha_j)} \frac{\Gamma(\alpha_0)}{\Gamma(\alpha_0+p+q)}$, therefore
\begin{align}
\mathbb{E}[\theta_i \theta_j] & = \frac{\alpha_i\alpha_j}{\alpha_0(\alpha_0+1)}\\
\mathbb{E}[\theta_i^2 \theta_j] & = \frac{\alpha_i(\alpha_i+1)\alpha_j}{\alpha_0(\alpha_0+1)(\alpha_0+2)}
\end{align}
\paragraph{trivariate moments} We know that $\mathbb{E}[\theta_i\theta_j\theta_k] = \frac{\Gamma(\alpha_i+1)}{\Gamma(\alpha_i)}\frac{\Gamma(\alpha_j+1)}{\Gamma(\alpha_j)}\frac{\Gamma(\alpha_k+1)}{\Gamma(\alpha_k)} \frac{\Gamma(\alpha_0)}{\Gamma(\alpha_0+3)}$, therefore
\begin{align}
\mathbb{E}[\theta_i \theta_j\theta_k] & = \frac{\alpha_i\alpha_j\alpha_k}{\alpha_0(\alpha_0+1)(\alpha_0+2)}
\end{align}
Therefore we obtain 
\begin{align}
\mathbb{E}[\theta] &= \frac{1}{\alpha_0} \alpha \\
\mathbb{E}[\theta\otimes\theta] &= \frac{1}{\alpha_0(\alpha_0+1)}[\alpha\otimes \alpha + \sum\limits_{t=1}^T \alpha_t e_t \otimes e_t]\\
\mathbb{E}[\theta\otimes\theta\otimes\theta] &= \frac{1}{\alpha_0(\alpha_0+1)(\alpha_0+2)}[\alpha\otimes\alpha\otimes\alpha 
+ \sum\limits_{t=1}^T \alpha_t e_t \otimes e_t\otimes \alpha \nonumber\\
+& \sum\limits_{t=1}^T \alpha_t \alpha \otimes e_t\otimes e_t 
+ \sum\limits_{t=1}^T \alpha_t e_t \otimes \alpha\otimes e_t
+ 2  \sum\limits_{t=1}^T \alpha_t e_t \otimes e_t\otimes e_t]
\end{align}

%
%
%
%


\begin{lemma}[Correctness of Method of Moments in Learning LDA~\cite{anandkumar2012spectral}]\label{lm:SampleCompLDA} Applying the method of moments over a corpus of $N$ documents sampled iid. There exist universal constants $C_1, C_2 \ge 0$ such that if $N > C_1 ((\alpha_0+1)/p^2_{\mathsf{min}} \sigma_k(\mu)^2)$, then $\norm{\mu_i-\hat{\mu_i}}_2 \leq C_2 \frac{(\alpha_0+1)^2 k^3}{p^2_{\mathsf{min}} \sigma_k(\mu) \sqrt{N} }$, where $p_{\mathsf{min}}=\min_{i} \frac{\alpha_i}{\alpha_0}$, $\mu$ is a matrix of stacked word-topic vectors, i.e. $\mu=[\mu_1|\dots|\mu_k]$.
\end{lemma}


\section{I. Sensitivity Proofs}\label{app:sensitivity_proofs}
\begin{table*}
	\centering
\begin{tabular}{l|c|c}
Config. &&\\
(Edge Set) & Sensitivity & Utility Loss ($\norm{\mu_i - \mu_{i}^{\mathsf{DP}}}$)   \\
\hline
{\bf 1:} {\small $\widehat{M}_3, \widehat{W}, \widehat{W}^\dag$}& $O(\frac{1}{N})$ &
$O(\frac{\sqrt{\sgm{1}k}}{\ds} ((\frac{\sqrt{d}}{N
  \sgm{k}^{1.5}}{\tdp{4}})^3+$  
\\

($e_3, e_4, e_8$) &&$\frac{\sqrt{d}}{N \sgm{k}^{1.5}}
\tdp{3}) + \frac{\sqrt{\sgm{1} d}}{\sgm{k} N} \tdp{8}+$
\\

&&
$\sqrt{\sgm{1}+\frac{\sqrt{d}}{N} \tdp{8}} \frac{\sqrt{k}}{\ds}$ \\
&& $\Big[
  (\frac{\sqrt{d}}{N \sgm{k}} \tdp{4})^3 + \frac{\sqrt{d}}{N
    \sgm{k}^{1.5}} \tdp{3} \Big])$\\

\hline

{\bf 2:} $\widehat{\mytensor{T}},\widehat{W}^\dag$& $O(\frac{k^{1.5}}{N (\sigma_k(\slda))^{1.5}})$ & $O(\frac{\sqrt{\sigma_1(\slda)
    k^{2.5}}}{\ds N \sk^{1.5}} \tdp{6} + \frac{\sqrt{\sigma_1(\slda)
    d}}{\sk N} \tdp{8} $  \\

($e_6, e_8$)&& $+ \sqrt{\sgm{1} + \frac{\sqrt{d}}{N} \tdp{8}}
\frac{k^{2.5} \tdp{6}}{\ds N \sgm{1}^{1.5}})$ \\

\hline

{\bf 3:} $\mbar, \widehat{W}^\dag$& $O(\frac{k^2}{\ds N (\sigma_k(\slda))^{1.5}})$  & $O(\frac{\sqrt{\sigma_1(\slda)
    k^{2.5}}}{\ds N \sk^{1.5}} \tdp{7} + \frac{\sqrt{\sigma_1(\slda)
    d}}{\sk N} \tdp{8} $    \\
    
($e_7,e_8$)& 
  &  $+\sqrt{\sgm{1}+ \frac{\sqrt{d}}{N} \tdp{8}}
\frac{k^{2} \tdp{7}}{\ds N \sgm{1}^{1.5}})$ \\

    

\hline


{\bf 4:} $\widehat{\mu}$ ($e_9$) & $O(\frac{k^2 \sqrt{\sigma_1(\slda)}}{\ds N {(\sigma_k(\slda))}^{1.5}})$
 &
$O(\frac{\sqrt{\sgm{1}d}k^2}{\ds N \sk^{1.5}} \tdp{9})$  \\

\end{tabular}
\caption{Utility of different configurations that guarantee
  differentially private LDA. The table lists edges in
  Figure~\ref{fig:flow} on which to add Gaussian noise in order to
  achieve differentially private topic model using method of
  moments. 
  We use $\tau$ as defined in Proposition~\ref{dp_prop} to decompose the dependence of the noise variance on both the sensitivity and the privacy parameters $\epsilon,\delta$, i.e. $\sigma = \Delta  \tdp{i,i}$, where $\Delta$ is the sensitivity. 
}\label{tab:configs}
\end{table*}

In proving the sensitivities for $\slda$ and $\tlda$ we rely on the fact that frequently in the calculations, we encounter probability vectors, matrices, and tensors where the elements sum to 1. This is identical to the stating that the $l_1$ norm equals 1. Further, we note the following Lemma which essentially states that taking the outer product of a vector with a probability vector or probability matrix does not increase the $\l_q$ norm of the vector and in fact keeps it the same if $q=1$. 

\begin{lemma}[Multiplying by probabilities does not change the norm]\label{pstruc}
Let $v_{p},M_{p}$ be a probability vector, matrix, respectively and let $v,u$ be ordinary vectors, matrices, respectively. Then the following holds:
\begin{align}
\norm{u v^{T}_p}_q  \leq \norm{u}_q \text{, which is equal if } q =1. 
\end{align}
\begin{align}
\text{If } T=M_p \otimes u \text{, then}
\norm{T}_q = \norm{M_p \otimes u}_q \leq \norm{u}_q  \text{, which is equal if } q =1. 
\end{align}
\end{lemma}

\begin{proof}
\begin{align}
\norm{u v^{T}_p}_q = (\sum_{i,j} |u_i {v_p}_j|^q)^{1/q} = (\sum_{i} |u_i|^q  \sum_{j} |{v_p}_j|^q)^{1/q} = \norm{v}_q \norm{u}_q \leq \norm{u}_q.
\end{align}
Where we used the fact that $\norm{x}_1 \ge \norm{x}_q$ for any $q \ge 1$ and that $\norm{v_p}_1=1$. Thus the above inequality is tight if $q=1$.
\begin{align}
\norm{T}_q = \norm{M_p \otimes u}_q = \Big(\sum_{i,j,k} |{M_p}_{i,j} {u}_k|^q\Big)^{1/q} =  \Big(\sum_{k} |u_k|^q \sum_{i,j}|{M_p}_{i,j}|^q\Big)^{1/q} = \norm{u}_q \norm{M_p}_q \leq \norm{u}_q.
\end{align}
Where we used the fact that for any matrix $M$, $\norm{M}_1 \ge \norm{M}_q$ for any $q \ge 1$\astfootnote{These norms are obtained by extending the vector definition to matrices or simply vectorizing the matrix and then calculating the norm.} and that $\norm{M_p}_1=1$. Thus the above inequality is tight if $q=1$.
\end{proof}

\begin{proposition} \label{prob_arrays}
\leavevmode 
\makeatletter
\@nobreaktrue
\makeatother
\begin{itemize}
\item $\mpf{n}$ \textup{is a probability vector}.
\item $\mps{n}$ \textup{is a probability matrix}.
\item $\mpt{n}$ \textup{is a probability tensor}.
\end{itemize}
\end{proposition}

\begin{proof}
The proof is immediate as these moments correspond to join probability estimates \cite{zou2013contrastive}, specifically: 
\begin{align*}
\mpf{n}(i) = \Pbb[x_1=i]
\end{align*}
\begin{align*}
\mpf{n}(i,j)= \Pbb[x_1=i,x_2=j]
\end{align*}
\begin{align*}
\mpf{n}(i,j,k)= \Pbb[x_1=i,x_2=j,x_3=k]
\end{align*}
\end{proof}

\subsection{Proof for Theorem \ref{m23} (sensitivity for $\slda$)}\label{sslda_proof}
 Let $\sslda$ be the $l_1$ sensitivity for $\slda$, then $\sslda$ is $\frac{2}{N}  + \frac{\alpha_0}{\alpha_0+1} \frac{4}{N}$=$O(\frac{1}{N})$. 
\begin{proof}
Let $\slda$ and $\nslda$ be two second order LDA moments generated from two neighboring corpora, WLOG assume the difference is in the $n^{th}$ record, i.e. $D=[c_1|\dots|c_{N-1}|c_N]$ and $D'=[c_1|\dots|c_{N-1}|c'_{N}]$ then:
\begin{align*}
\slda - \nslda =  &\frac{1}{N} (\mps{N} - {\mps{N}}') - \frac{a}{2\binom{N}{2}}\ \Bigg(\bigg[ \mpf{N} \otimes \bigg( \sum_{n=1}^{N-1}  \mpf{n} \bigg) + \bigg( \sum_{n=1}^{N-1} \mpf{n} \bigg) \otimes \mpf{N}  \bigg]   \\
& - \bigg[ {\mpf{N}}' \otimes \bigg( \sum_{n=1}^{N-1}  \mpf{n} \bigg) + \bigg( \sum_{n=1}^{N-1} \mpf{n} \bigg) \otimes {\mpf{N}}'  \bigg] \Bigg) \\
=& \frac{1}{N} (\mps{N} - {\mps{N}}') -  \frac{a}{2\binom{N}{2}}\ \Bigg( ( \mpf{N} - {\mpf{N}}') \otimes \bigg( \sum_{n=1}^{N-1}  \mpf{n} \bigg) + \bigg( \sum_{n=1}^{N-1}  \mpf{n} \bigg) \otimes ( \mpf{N} - {\mpf{N}}')   \Bigg)  \\
=& \frac{1}{N} (\mps{N} - {\mps{N}}') - \frac{a}{N}\ \Bigg( ( \mpf{N} - {\mpf{N}}') \otimes \bigg( \frac{1}{N-1}\sum_{n=1}^{N-1}  \mpf{n} \bigg) + \bigg( \frac{1}{N-1} \sum_{n=1}^{N-1}  \mpf{n} \bigg) \otimes ( \mpf{N} - {\mpf{N}}')  \Bigg) 
\end{align*}  
Note that according to proposition (\ref{prob_arrays}) ${\mpf{N}}$ and ${\mpf{N}}'$ are probability vectors and ${\mps{N}}$ and ${\mps{N}}'$ are probability matrices. Further, $\bigg(\frac{1}{N-1}\sum_{n=1}^{N-1}  \mps{n}\bigg)$ is also a probability matrix since it's the normalized sum of probability matrices. We upper bound the $l_1$ norm of the expression by applying the triangular inequality and using lemma (\ref{pstruc}) for the terms involving a tensor product. This leads to the following:
\begin{align*}
&\norm{\slda - \nslda}_1 \leq \frac{2}{N}  + \frac{4a}{N}  = \frac{2}{N}  + \frac{\alpha_0}{\alpha_0+1} \frac{4}{N}  = O(\frac{1}{N}) \\
\end{align*}  
 $a$ was replaced by its expression as in the above $a =\frac{\alpha_0}{\alpha_0+1}$  in the above. 
\end{proof}

\subsection{Proof for Theorem \ref{m23} (sensitivity for $\tlda$)}\label{appendix:sensitivity-tlda} Let $\stlda$ be the $l_1$ sensitivity for $\tlda$, then $\stlda$ is $\frac{2}{N} + \frac{4\alpha_0}{\alpha_0+2} \frac{1}{N} + \frac{12\alpha_0^2}{(\alpha_0+1)(\alpha_0+2)}\frac{(N-1)}{N(N-2)}= O(\frac{1}{N})$.

\begin{proof}
Following a similar setting as in \ref{sslda_proof} we have the two moments $\tlda$ and $\tlda'$ generated from two neighboring corpora. First we note that the expression of $\tlda$ and $\ntlda$ have the following form: $\frac{1}{N} \sum_{n=1}^{N} \tilde{\tilde{M}}_3^n  + \mathbf{B}_1 + \mathbf{B}_2 + \mathbf{B}_3 + \mathbf{b}$. Effectively there are three kinds of terms: \textbf{(a)} $\frac{1}{N} \sum_{n=1}^{N} \tilde{\tilde{M}}$, \textbf{(b)} $\mathbf{B}_1 $, and \textbf{(c)}$\mathbf{b}$. Since $\mathbf{B}_2$ and $\mathbf{B}_3$are permuted versions of $\mathbf{B}_1$ they have a similar behavior 
\paragraph{(a) $\frac{1}{N} \sum_{n=1}^{N} \tilde{\tilde{M}}$:}
The first term difference between $\tlda$ and $\ntlda$ would result in $\frac{1}{N} (\mpt{N} - {\mpt{N}}')$. \begin{align*}
\frac{1}{N} \norm{\mpt{N} - {\mpt{N}}'}_1 \leq \frac{1}{N} \Bigg(\norm{\mpt{N}}_1+\norm{{\mpt{N}}'}_1 \Bigg) \leq \frac{2}{N}
\end{align*}
Note that both ${\mpt{N}}$ and ${\mpt{N}}'$ are probability tensors.

\paragraph{(b) $\mathbf{B}_1$:} Based on the minimized expression, the $\mathbf{B}_1$ term difference between $\tlda$ and $\tlda'$ is equal to:
\begin{align*} 
\mathbf{B}_1 - \mathbf{B'}_1 = &\frac{b}{2\binom{N}{2}}   \Bigg[  \mps{N}  \otimes \Big(\sum\limits_{n=1}^{N-1} \mpf{n} \Big)  + \Big(\sum\limits_{n=1}^{N-1} \mps{n}  \Big) \otimes \mpf{N} \\
& - {\mps{N}}'  \otimes \Big(\sum\limits_{n=1}^{N-1} \mpf{n} \Big)  - \Big(\sum\limits_{n=1}^{N-1} \mps{n}  \Big) \otimes {\mpf{N}}'  \Bigg] \\
=& \frac{b}{N}\Bigg[ \Big(\mps{N}-{\mps{N}}'\Big)  \otimes \Big(\frac{1}{N-1}\sum\limits_{n=1}^{N-1} \mpf{n} \Big)  + \Big(\frac{1}{N-1}\sum\limits_{n=1}^{N-1} \mps{n}  \Big) \otimes \Big(\mpf{N}-{\mps{N}}\Big)'\Bigg] 
\end{align*} 
Note that $\frac{1}{N-1}\sum\limits_{n=1}^{N-1} \mpf{n}$ and $\frac{1}{N-1}\sum\limits_{n=1}^{N-1} \mps{n}$ are probability vectors and matrices, respectively. Thus lemma \ref{pstruc} can be used to upper bound the $l_1$ norm, leading to the following:
\begin{align*} 
\norm{\mathbf{B}_1 - \mathbf{B'}_1}_1 \leq \frac{|b|}{N} \big(2+2\big) = \frac{4|b|}{N} = \frac{4\alpha_0}{\alpha_0+2} \frac{1}{N}
\end{align*}

\paragraph{(c) $\mathbf{b}$:} Based on the minimized expression, the $\mathbf{b}$ term difference between $\tlda$ and $\tlda'$ is equal to:
\begin{align*} 
\mathbf{b}-{\mathbf{b}}' = & \frac{c}{6\binom{N}{3}} \Bigg[ \Bigg(\mpf{N} \otimes \big(\sum_{\substack{m=1,p=1 \\ \text{distinct}}}^{N-1} \mpf{m} \otimes \mpf{p}  \big)  +  \big(\sum_{\substack{n=1,p=1 \\ \text{distinct}}}^{N-1} \mpf{n} \otimes \mpf{N} \otimes \mpf{p}\big) \\ &  + \big(\sum_{\substack{n=1,m=1 \\ \text{distinct}}}^{N-1} \mpf{n} \otimes \mpf{m} \big) \otimes \mpf{N} \Bigg)- \Bigg({\mpf{N}}' \otimes \big(\sum_{\substack{m=1,p=1 \\ \text{distinct}}}^{N-1} \mpf{m} \otimes \mpf{p}  \big)  \\&-  \big(\sum_{\substack{n=1,p=1 \\ \text{distinct}}}^{N-1} \mpf{n} \otimes {\mpf{N}}' \otimes \mpf{p}\big)  - \big(\sum_{\substack{n=1,m=1 \\ \text{distinct}}}^{N-1} \mpf{n} \otimes \mpf{m} \big) \otimes \mpf{N} \Bigg) \Bigg] \\
=& \frac{c(N-1)}{N(N-2)} \Bigg[ \big(\mpf{N}-{\mpf{N}}'\big)  \otimes \big(\frac{1}{(N-1)^2}\sum_{\substack{m=1,p=1 \\ \text{distinct}}}^{N-1} \mpf{m} \otimes \mpf{p}  \big)  \\&+  \big(\frac{1}{(N-1)^2}\sum_{\substack{n=1,p=1 \\ \text{distinct}}}^{N-1} \mpf{n} \otimes \big(\mpf{N}-{\mpf{N}}'\big) \otimes \mpf{p}\big) \\ &  + \big(\frac{1}{(N-1)^2}\sum_{\substack{n=1,m=1 \\ \text{distinct}}}^{N-1} \mpf{n} \otimes \mpf{m} \big) \otimes \big(\mpf{N}-{\mpf{N}}'\big) \Bigg]
\end{align*} 
Similarly, we have probability tensors so we use lemma \ref{pstruc} to bound the $l_1$ norm. This results in:
\begin{align*} 
\norm{\mathbf{b}-{\mathbf{b}}'}_1\leq\frac{c(N-1)}{N(N-2)} (2+2+2) = \frac{6c(N-1)}{N(N-2)} = \frac{12\alpha_0^2}{(\alpha_0+1)(\alpha_0+2)}\frac{(N-1)}{N(N-2)}
\end{align*} 

Combing the results from \textbf{(a), (b)} and \textbf{(c)}, we have the following bound:
\begin{align*} 
\stlda \leq \frac{2}{N} + \frac{4\alpha_0}{\alpha_0+2} \frac{1}{N} + \frac{12\alpha_0^2}{(\alpha_0+1)(\alpha_0+2)}\frac{(N-1)}{N(N-2)} = O(\frac{1}{N})
\end{align*} 

\end{proof}

\subsection{Proof for Theorem \ref{swt} (sensitivity for $\wt$ )} As explained before, the whitened tensor is denoted as $\wtshort$ for simplicity. Therefore we denote the sensitivity of $\wt$ as $\swt$.  Theorem~\ref{swt} states that $\swt= O(\frac{k^{3/2}}{N\sgm{k}^{3/2}})$.

We need the following Lemma to prove Theorem~\ref{swt}. 

\begin{lemma} \label{wdist}
$\norm{\nw-\w}_F \leq \frac{\sqrt{2k} \sslda}{{\sigma_k(\slda) \sqrt{\frac{1}{2} (\sgm{k}+\sgm{k+1})}}} $
\end{lemma}
\begin{proof}
We follow an analysis similar to \cite{anandkumar2012spectral}. Note that the whitening matrix $\w$ is defined such that:
 $$\w^T \sldak \w = I.$$ 
Analogously for the neighboring corpus,
$$\nw^T \nsldak \nw = I.$$

Let 
$E_{M_2}$ denote the perturbation introduced to $\slda$ by changing a single record.

Because the spectral gap of the perturbation introduced by modifying a single record is small according to the condition, applying the original whitening matrix to the neighboring data base moment $\nslda$ would lead to a rank $k$ matrix of size $k\times k$. 

Therefore, $\w^T \nsldak \w$ is a rank $k$ matrix of size $k\times k$, which can be factorized as: 
$$\w^T \nsldak \w = A D A^{T}$$
where $A$ are the singular vectors of $\w^T \nsldak \w$, and $D$ is a diagonal matrix of the corresponding singular values of $\w^T \nsldak \w$.
This also leads to $\nw = \w A D^{\frac{-1}{2}} A^{T}$. Using this, we observe: 
\begin{align}
    \norm{\nw-\w}
    & = \norm{\nw - \nw A D^{\frac{1}{2}} A^{T}} \\
    & = \norm{\nw ( I - A D^{\frac{1}{2}} A^{T})} \nn \\
    & \leq  \norm{\nw} \norm{I - A D^{\frac{1}{2}} A^{T}} \nn \\
\end{align} 
Now we bound $\norm{I - A D^{\frac{1}{2}} A^{T}}$:
\begin{align}
    \norm{I - A D^{\frac{1}{2}} A^{T}}
    & = \norm{A^{T} A - \nw A D^{\frac{1}{2}} A^{T}} \\
    & = \norm{I-D^{\frac{1}{2}}} \nn \\
    & \leq  \norm{(I-D^{\frac{1}{2}})(I+D^{\frac{1}{2}})} \nn \\
    & \leq  \norm{(I-D)} \nn \\
    & = \norm{I-A D A ^{T}} \nn \\ 
    & = \norm{\w^T \sldak \w  - \nw^T \nsldak \nw  } \nn \\ 
    & \leq \norm{\w}^2 \norm{\sldak-\nsldak} \nn \\ 
    & \leq \norm{\w}^2  \norm{E_{M_2}} 
\end{align} 
We know that
\begin{align}
\norm{\w}^2 & \le \frac{1}{\sigma_k(\slda)}\\ 
\norm{\nw} & \le \frac{1}{\sqrt{\sigma_k(\nslda)}} \leq  \frac{1}{\sqrt{\sgm{k}-\norm{E_{M_2}}_2}} \leq \frac{1}{\sigma_k(\slda) \sqrt{\frac{1}{2} (\sgm{k}+\sgm{k+1})}}  
\label{eq:w_hat_prime_norm}
\end{align}
Weyl's theorem was used in the last bound in Equation~\eqref{eq:w_hat_prime_norm}. Bounding the Frobenius norm, would result in the following: 
$$\norm{\nw-\w}_F \leq \sqrt{2k} \norm{\nw -\w} \leq \frac{\sqrt{2k} \norm{E_{M_2}}}{\sigma_k(\slda) \sqrt{\sigma_k(\slda)-\norm{E_{M_2}}}} \leq \frac{\sqrt{2k} \sslda}{\sigma_k(\slda) \sqrt{\frac{1}{2} \sgm{k}+\sgm{k+1}}},$$
where we have used the fact that the $l_1$ norm upper bounds the spectral norm of a matrix, since it upper bounds the Frobenius. 
\end{proof}

Now we are ready to prove Theorem~\ref{swt}.
\begin{proof} $\ntlda = \tlda + E_3$.
\begin{align}
\norm{\wt - \nwt}_F  & = \lVert\wt - \hat{M}_3^{LDA}(\hat{W'},\hat{W'},\hat{W'}) \nn\\
&- E_3(\hat{W'},\hat{W'},\hat{W'}) \rVert_F \\
& \leq \norm{\hat{M}_3^{LDA}(\hat{W-W'},\hat{W-W'},\hat{W-W'})}_F \nn\\
&+ \norm{E_3(\hat{W'},\hat{W'},\hat{W'})}_F \\
& \leq \norm{\tlda}_F \norm{\w - \nw}_F^3 + \norm{\stlda}_F \norm{\nw}_F^3 
\end{align}

We have used the fact that the Frobenius norm of the difference between the tensors is bounded above by the $l_1$ norm of the difference $\stlda$. {To bound the $l_1$ norm of $\tlda$ we use  an analysis similar to calculating $\stlda$. Again we note that the $l_1$ norm upper bounds the Frobenius norm:}
\begin{align}
    \norm{\tlda}_F \leq \norm{\slda}_1 = 1 + \frac{6 \alpha_0}{\alpha_0+2} \frac{N}{N-1} + \frac{6 \alpha_0^2}{(\alpha_0+1)(\alpha_0+2)} \frac{N^3}{N(N-1)(N-2)}
\end{align}
Combining all the expressions we get:
\begin{align}
 \swt &= \norm{\wt - \nwt}_F   \\
& \leq ( 1 + \frac{6 \alpha_0}{\alpha_0+2} \frac{N}{N-1} + \frac{6 \alpha_0^2}{(\alpha_0+1)(\alpha_0+2)} \frac{N^3}{N(N-1)(N-2)})\nn \\
& \times \frac{(2k)^{3/2} (\sslda)^3}{(\sgm{k} \sqrt{\frac{1}{2} (\sgm{k}+\sgm{k+1})})^3} +  \frac{\stlda k^{3/2}}{(\frac{1}{2} (\sgm{k}+\sgm{k+1}))^{3/2}} \\
&= O(\frac{k^{3/2}}{N\sgm{k}^{3/2}})\\
\end{align}
We see that if $N$ is larger than $d^{3/2}$, then $N \sgm{k}^{3/2} \ge 1$  as $\sgm{i}$ is in the order of $1/d$.
\end{proof}

\subsection{Proof for Theorem~\ref{svbw} (sensitivity of the output of tensor decomposition $\bar{\mu_i},\bar{\alpha}_i$ )}
\label{app:svbw} 

Let $\bar{\mu}_1,\dots,\bar{\mu}_k $ and $\bar{\alpha}_1,\dots,\bar{\alpha}_k$ be the results of tensor decomposition before unwhitening.
The sensitivity of $\bar{\mu}_i$, denoted as $\svbw$, and the sensitivity of $\bar{\alpha}_i$, denoted as $\sabw$, are both upper bounded by $\svbw \leq O(\frac{k^2}{\ds N  (\sigma_k(\slda))^{3/2}})$, where $\ds=\min_{i \in [k]} \frac{\sigma_i-\sigma_{i+1}}{4}$, $\sigma_i$ is the $i^{th}$ eigenvalue of $\wt$. 

\begin{proof}
The proof follows from the result of the simultaneous tensor power method (Theorem 1 in \cite{wang2017tensor}). Replacing the original eigenvectors with those resulting from database $D$ leads to tensor $\wt$, then the tensor resulting from corpus $D'$ with one record changed yields $\nwt$ where the spectral norm of the error is upper bounded by $\epsilon$, if $\swt$ is sufficiently small $\swt \leq \frac{\ds \epsilon}{2\sqrt{k}}$ . Therefore we get 
$\norm{\bar{\mu_i}-\bar{\mu'_i}}_2 \leq \frac{2\sqrt{k}\swt}{\ds}$ and $|\bar{\alpha}_i-\bar{\alpha'}_i| \leq \frac{2\sqrt{k}\swt}{\ds}.$
\end{proof}

\subsection{Proof for Theorem~\ref{svaw} (sensitivity of the final output  $\mu_i,\alpha_i$)} \label{save}

We now prove the sensitivity of the final output  $\mu_i,\alpha_i$: $\sv= O(\frac{k^2 \sqrt{\sigma_1(\slda)}}{\ds N \sigma^{3/2}_k(\slda)}).$  

\begin{proof}
We point out a number of things. Tensor decomposition outputs are: $\bar{\mu}_i, \bar{\alpha}_i, i \in [k]$, where, $\bar{\alpha}_i=\frac{2 \sqrt{(\alpha_0+1)\alpha_0}}{(\alpha_0+2)\sqrt{\alpha_i}}$. In order to recover the desired word topic vector $\mu$, we have to ``unwhiten'' to get the $\mu_i$ and $\alpha_i$ before whitening, i.e. $\mu_i = \frac{1}{\sqrt{\alpha^r_i}} (W^T)^\dagger\bar{\mu}_i$, where $\frac{1}{\sqrt{\ari{i}}}=\frac{(\alpha_0+2)}{2\sqrt{(\alpha_0+1)\alpha_0}} \bar{\alpha}_i$. The sensitivity would be:
\begin{align*}
\max_{D,D'} \norm{\mu_i -\mu\prime_i}& \leq \max_{D,D'} \Bigg\{ \norm{ \frac{1}{\sqrt{\alpha^r_i}} (W^T)^\dagger \mbar{i} - \frac{1}{\sqrt{\alpha^{r,'}_i}}  (W^{T,'})^\dagger\mbar{i}^\prime }_2\Bigg\}  \\
& \leq \max_{D,D'} \Bigg\{ \frac{1}{\sqrt{\alpha^r_i}} \norm{(W^T)^\dagger} \norm{ \bar{\mu}_i -\bar{\mu}^\prime_i}  
+ \frac{1}{\sqrt{\alpha^r_i}} \norm{W^\dagger-(W^\prime)^\dagger} + \norm{(W^\prime)^\dagger} |\frac{1}{\sqrt{\alpha^r_i}}-\frac{1}{\sqrt{\alpha^r_{i,\prime}}}| \Bigg\} 
\end{align*}
We note the following:
\begin{enumerate}
\item $\max_{D,D'} |\frac{1}{\sqrt{\alpha^r_i}}-\frac{1}{\sqrt{\alpha^r_{i,\prime}}}| = \max_{D,D'} |\ax \abar{i}- \ax \abar{i}^\prime| \leq \ax \max_{D,D'} |\abar{i}-\abar{i}^\prime| \leq \ax \frac{2 \sqrt{k} \swt}{\ds}$,
where the above follows from the simultaneous power iteration method. 
\item $\max_{i \in [k]} \frac{1}{\sqrt{\alpha^r_i}} \leq \ax  \max_{i \in [k]} \abar{i} = \ax \sigma_1(\wtshort)$ 
\item $\max \norm{((W\prime)^T)\dagger} \leq \sqrt{\sigma_1(\nslda)} \leq \sqrt{\sgm{1}+\sslda} $ 

\item Following an analysis similar to that in \ref{wdist}, we obtain $\norm{W^\dagger-(W^{\prime})^\dagger} \leq \frac{\sqrt{\sgm{1}}}{\sgm{k}} \sslda $.
\end{enumerate}

Combining all of this together leads to the following 
\begin{align*}
\max_{D,D'} \norm{\mu_i -\mu\prime_i}&\leq 
 \ax \sigma_1(\wtshort) \sqrt{\sgm{1}} \frac{2\sqrt{k} \swt}{\ds}\nn\\
 & + 
\ax  \sigma_1(\wtshort)  \frac{\sqrt{\sgm{1}} }{\sgm{k}} \sslda \\
&+ \ax \sqrt{\sgm{1}+\sslda}  \frac{2\sqrt{k} \swt}{\ds}  \\
&= O(\frac{k^2 \sqrt{\sigma_1(\slda)}}{\ds N \sigma^{3/2}_k(\slda)})
\end{align*}

\end{proof}

\subsection{Proof for Lemma~\ref{lm:local-sensitivity}}\label{sec:local-sensitivity}
\begin{proof}
	Let $x,x'$ be two adjacent data sets and the overall output be $O := f(\text{DATA}) + Z(\epsilon_2,\delta_2, \tilde{LS})$. Let $S_1\subset \mathrm{Dom}(O), S_2\subset \mathrm{Dom}(\mathrm{LS})$ be any measurable sets. 
	
	Let $E$ be the measurable set of $\tilde{\mathrm{LS}}$ that represents the event that $\tilde{\mathrm{LS}} \geq \mathrm{LS}$.
	\begin{align*}
	&\Pr[(O, \tilde{\mathrm{LS}}) \in S_1\times S_2 |x ] \\
	=&  	\Pr[(O, \tilde{\mathrm{LS}}) \in S_1\times (S_2 \cap E) |x ] + \Pr[(O, \tilde{\mathrm{LS}}) \in S_1\times S_2 \cap E^c |x ]  \\
	\leq&     \Pr[(O, \tilde{\mathrm{LS}}) \in S_1\times (S_2 \cap E) |x ] +\delta_3\\
	\leq&  e^{\epsilon_1+\epsilon_2}  \Pr[(O, \tilde{\mathrm{LS}}) \in S_1\times (S_2 \cap E) |x' ] + \delta_1 + \delta_2 + \delta_3\\
	\leq&  e^{\epsilon_1+\epsilon_2}  \Pr[(O, \tilde{\mathrm{LS}}) \in S_1\times S_2 |x' ] + \delta_1 + \delta_2 + \delta_3
	\end{align*}
	The fourth line holds due to the fact that under event the $E$, $\tilde{\mathrm{LS}}$ is always a valid upper bound of the local sensitivity, therefore, conditioning on the $\sigma$-field induced by $E\cap S_2$ for any $S_2$, $O$ is an $(\epsilon_2,\delta_2)$-DP release. 
By the simple composition Theorem of $(\epsilon,\delta)$-DP \cite{dwork2014algorithmic}[Theorem B.1,], by taking the measurable set of interest to be $S_1\times (S_2 \cap E)$, we have that 
$$
\Pr[(O, \tilde{\mathrm{LS}}) \in S_1\times (S_2 \cap E) |x ]  \leq e^{\epsilon_1+\epsilon_2}  \Pr[(O, \tilde{\mathrm{LS}}) \in S_1\times (S_2 \cap E) |x' ] + \delta_1 + \delta_2
$$
which wraps up the proof.
\end{proof}

\subsection{Proof for Sensitivity of singular values $\sgm{k}$ (Lemma~\ref{lm:sensitivity-singularvalue-gap}) }\label{sec:global-sensitivity-m2}
\begin{proof}
	We first prove that the global sensitivity of $\sgm{k}$ is $1/n$. 
	By Weyl's lemma \cite{stewart1998perturbation}[Theorem 1], for any matrix $X$, any $i$, the singular value $|\sigma_i(X) - \sigma_i(X+E)|  \leq \|E\|_2$.  In our case, $E$ is coming from adding or removing one data point and we know that $\|E\|_2\leq \|E\|_F\leq \|E\|_{1,1}\leq 2/n$, hence the bound.

	Now we prove that the global sensitivity of $\ds =\min_{i \in [k]} \frac{\sigma_i(\wtshort)-\sigma_{i+1}(\wtshort)}{4}$. 
	For any tensor $\wtshort$, we consider a polyadic form or the so called tensor decomposition form, and denote the singular values as the amplitude of the components in the polyadic form. As shown in Section~\ref{appendix:sensitivity-tlda}, $|\sigma_i(\wtshort) - \sigma_i(\wtshort+\mytensor{E})|  \leq \|\mytensor{E}\| \le 1$, where $\mytensor{E}$ comes from adding or removing one data point. 

	
	\end{proof}

\section{J. Utility Proofs\label{app:utility}}
Before starting the utility proofs, we point out a number of things. Tensor decomposition outputs:$\bar{\mu}_i, \bar{\alpha}_i, i \in [k]$. Where, $\bar{\alpha}_i=\frac{2 \sqrt{(\alpha_0+1)\alpha_0}}{(\alpha_0+2)\sqrt{\alpha_i}}$. In order to recover the desired word topic vector $\mu$, we have to 'reverse whiten', i.e. $\mu_i = \frac{1}{\sqrt{\alpha^r_i}} (W^T)^\dagger\bar{\mu}_i$, where $\frac{1}{\sqrt{\ari{i}}}=\frac{(\alpha_0+2)}{2\sqrt{(\alpha_0+1)\alpha_0}} \bar{\alpha}_i$. We need to establish the distance between the non-differentially private output and the differentially private output, i.e. $\norm{\mu_i -\mu^{DP}_i} $. This can be upper bounded similar to \ref{save} by the following:
\begin{align}
\norm{\mu_i -\mu^{DP}_i} &\leq \frac{1}{\sqrt{\alpha^r_i}} \norm{(W^T)^\dagger} \norm{ \bar{\mu}_i -\bar{\mu}^{DP}_i}\nn  \\
&+ \frac{1}{\sqrt{\alpha^r_i}} \norm{W^\dagger-(W^{DP})^\dagger} + \norm{(W^{DP})^\dagger} |\frac{1}{\sqrt{\alpha^r_i}}-\frac{1}{\sqrt{\alpha^r_{i,DP}}}| 
\end{align} 
For this we frequently need to bound the following:
\begin{itemize}
    \item $\norm{\mbar{i}-\mbar{i}^{DP}}$
    \item $\norm{W^\dagger-(W^{DP})^\dagger}$
    \item $\norm{(W^{DP})^\dagger}$
    \item $|\frac{1}{\sqrt{\alpha^r_i}}-\frac{1}{\sqrt{\alpha^r_{i,DP}}}|$
    \item $|\abar{i}-\abar{i}^{DP}|$
\end{itemize}

We point out the following facts before preceding:
\begin{itemize}
\item $|\frac{1}{\sqrt{\alpha^r_i}}-\frac{1}{\sqrt{\alpha^r_{i,DP}}}| \leq |\frac{(\alpha_0+2)}{2\sqrt{(\alpha_0+1)\alpha_0}} \abar{i}- \frac{(\alpha_0+2)}{2\sqrt{(\alpha_0+1)\alpha_0}} \abar{i}^{DP} | \leq \frac{(\alpha_0+2)}{2\sqrt{(\alpha_0+1)\alpha_0}} |\abar{i}-\abar{i}^{DP}|$
\item $\norm{(W^T)^\dagger} \leq \sqrt{\sgm{1}}$
\item $\frac{1}{\sqrt{\alpha^r_i}} = \frac{(\alpha_0+2)}{2\sqrt{(\alpha_0+1)\alpha_0}} \abar{i} \leq \frac{(\alpha_0+2)}{2\sqrt{(\alpha_0+1)\alpha_0}} \sigma_1(\wtshort) $
\end{itemize}

\subsection{Perturbation on $\slda$ , $\tlda$  Config. 1 ($e_3,e_4,e_8$): Proof for Theorem~\ref{thm:utility-config1}}\label{sec:appendix-utility-config1}
Similar to the perturbation on ($e_6,e_8$). We have that
\begin{align}
\norm{W^\dagger-(W^{DP})^\dagger} \leq \frac{\sqrt{\sgm{1}} \norm{E_{8,G}}}{\sgm{k}}\\
\norm{(W^{DP})^\dagger} \leq \sqrt{\sgm{1}+\norm{E_{8,G}}} 
\end{align}
Now the perturbed tensor can be represented as $\tlda^{DP} = \tlda + E_{3,G}$, where $E_{3,G}$ is symmetric Gaussian noise that has been added to the original tensor. Similar to the sensitivity analysis for the whitened tensor, we have that the error $\Phi$ can be bounded as follows:
\begin{align}
\norm{\Phi}_2 &= \norm{\wt - \tlda^{DP}(W^{DP},W^{DP},W^{DP})}_2 \\ 
& \leq \norm{\tlda} \norm{W-W^{DP}}^3 + \norm{E_{3,G}} \norm{W^{DP}} 
\end{align}
Following an analysis similar to bounding  $\norm{W^\dagger-(W^{DP})^\dagger}$, we get that $\norm{W^\dagger-(W^{DP})^\dagger} \leq \frac{\norm{E_{8,G}}}{\sgm{k} \sqrt{\frac{\sgm{k}} {2}}}$.
According to \ref{gtb} we have that with high probability $\norm{E_{3,G}} = O(\sqrt{d} \stlda \tdp{3} )$. We note the following $\norm{\mbar{i}-\mbar{i}^{DP}}_2 \leq \frac{2\sqrt{k} \norm{\Phi}}{\ds}$ using the simultaneous power iteration of \cite{wang2017tensor}. Similarly we have $|\abar{i}-\abar{i}^{DP}| \leq \frac{2\sqrt{k} \norm{\Phi}}{\ds}$ and that $|\frac{1}{\sqrt{\alpha^r_i}}-\frac{1}{\sqrt{\alpha^r_{i,DP}}}| \leq \ax \frac{2\sqrt{k} \norm{\Phi}}{\ds}$. This leads to  $ \norm{\mu_i-\mu_i^{DP}}_2 \leq \ax \sigma_1(\wtshort) \sqrt{\sgm{1}} \frac{2\sqrt{k} \norm{\Phi}}{\ds} + \ax \sigma_1(\wtshort) \frac{\sqrt{\sgm{1}}}{\sgm{k}} \norm{E_{8,G}} + \sqrt{\sgm{1}+\norm{E_{8,G}}} \ax \frac{2\sqrt{k} \norm{\Phi}}{\ds}.$

Based on the bound on $\norm{\Phi}$ we have with high probability $ \norm{\mu_i-\mu_i^{DP}}_2 = O(\frac{\sqrt{\sgm{1}k}}{\ds} ((\frac{\sqrt{d}}{N \sgm{k}^{3/2}}{\tdp{4}})^3+ \frac{\sqrt{d}}{N \sgm{k}^{3/2}} \tdp{3})  + \frac{\sqrt{\sgm{1} d}}{\sgm{k} N} \tdp{8}+ \sqrt{\sgm{1}+\frac{\sqrt{d}}{N} \tdp{8}} \frac{\sqrt{k}}{\ds} \Big[ (\frac{\sqrt{d}}{N \sgm{k}} \tdp{4})^3 + \frac{\sqrt{d}}{N \sgm{k}^{3/2}} \tdp{3} \Big]).$

\subsection{Perturbation on $\wtshort$ and $\slda$ Config. 2($e_6,e_8$): Proof for Theorem~\ref{thm:utility-config2}}\label{sec:appendix-utility-config2}

This configuration has two properties: the noise level introduced is
low because the whitening step reduces the tensor dimension from
$\tlda \in \Rbb^{d\times d\times d}$ to $\wtshort = \wt \in \Rbb^{k\times k\times
  k}$.  However, even though the dimension of the tensor is reduced,
unless the whitening tensor (resulting from eigendecomposition
over $\slda$) is stable, the sensitivity of the whitened tensor is not
necessarily low.

Note that the sensitivity of $\slda$ falls with $\frac{1}{N}$
(Theorem \ref{m23}).  Therefore, we expect the sensitivity of $\wt$ to
drop with an increasing number of records. 
As Theorem \ref{swt} states, $\swt=O(\frac{k^{3/2}}{N
  \sigma^{3/2}_k(\slda)})$, if $\sslda \leq \sgs$. Thus, given the
spectral gap requirement, the sensitivity of the whitened tensor is
$\swt$.

$\slda$ is used to generate both the whitening and unwhitening matrix, and unlike input perturbation, the sensitivity
over $\slda$ and $\tlda$
falls as the dataset size increases (Theorem \ref{m23}). 
However, an issue with this configuration is that adding noise to $\tlda$ leads to higher
noise build up prior to the tensor decomposition. Note that by (\ref{gtb}) w.h.p the norm of the error is $O(\sqrt{d} \sigma)$, with $\sigma$ being the variance of the noise (this bound would be $\sqrt{k} \sigma$ if the noise is added to a symmetric tensor of size $k$). Tensor
decomposition methods, in particular \cite{wang2017tensor} require the
spectral norm of the {perturbation to the tensor} to be lower than a certain
threshold. 
Following arguments similar to \cite{wang2016online}, the spectral norm of the error is $O(\frac{\sqrt{d}}{N\epsilon_{3}})$ and should be below $\frac{\sqrt{k}}{\ds \sigma_k(\wtshort) }$. Thus $\epsilon_{3}$ should satisfy  $\epsilon_{3}=\Omega(\frac{\sqrt{kd}}{\ds \sigma_k(\wtshort) N})$  to establish utility guarantees for tensor decomposition. Following similar arguments, this time using the bound on the spectral norm of the noisy matrices, to guarantee utility, the 
differentially private whitening $W$ and pseudo-inverse $W^{\dagger}$ should be close to their non-differentially private values,
 which requires both
$\epsilon_{4}$ and $\epsilon_{8}$ to be
$\Omega(\frac{\sqrt{d}}{(\sgs N)})$. 
Although, the privacy parameters have a lower bound of $\sqrt{d}$, the bound also falls with $\frac{1}{N}$. 

The spectral norm of the noise added to $\slda$ can be bounded by \ref{gmb} to be $O(\frac{\sqrt{d}}{N}\tdp{8})$ with high probability. Now, if we have $N=\Omega(\frac{\sqrt{d}\tdp{8}}{\sgm{k}-\sgm{k+1}})$, then with w.h.p we have that $\norm{E_{8,G}} \leq \frac{\sgm{k}-\sgm{k+1}}{2}$, where $\norm{E_{8,G}}$ is the spectral norm of the Gaussian matrix. This condition enables us to bound  $\norm{W^\dagger-(W^{DP})^\dagger}$, in a manner similar to establishing the bounds between $\norm{W-W'}$ in \ref{wdist}. Following a similar analysis, given that
\begin{align}
W^T (\slda)_{k} W &= I, \\
W^{T,DP} (\slda+E_{8,G})_{k} W^{DP} &= I,\\
W^T (\slda+E_{8,G})_k W & =ADA^T,
\end{align}
we have that 
$\norm{W^\dagger-(W^{DP})^\dagger} \leq \norm{W^\dagger} \norm{I-D}$. We know that $\norm{W^\dagger} \leq \frac{1}{\sqrt{\sgm{k}}}$ and  $\norm{I-D}$ can be bounded as follows:
\begin{align*}
\norm{I-D} & \leq \norm{I-ADA^T} \\ & \leq \norm{W^T (\slda)_{k} W - W^T (\slda+E_{8,G})_k W} \\ & \leq \norm{W}^2 \norm{(\slda)_{k}-(\slda+E_{8,G})_{k}} \\ & \leq \norm{W}^2 \norm{E_{8,G}} \\ & \leq \frac{\norm{E_{8,G}}}{\sgm{k}} 
\end{align*}
This leads to
$\norm{W^\dagger-(W^{DP})^\dagger} \leq \frac{\sqrt{\sgm{1}} \norm{E_{8,G}}}{\sgm{k}}$.

Moreover, it is immediate by Weyl's theorem that $\norm{(W^{DP})^\dagger} \leq \sqrt{\sigma_1(\slda +E_{8,G})} \leq \sqrt{\sgm{1}+\norm{E_{8,G}}}$. 

Finally, by the results of simultaneous power iteration (with an argument similar to Theorem \ref{svbw}),  if $N$ is sufficiently large, we have that $\norm{\mbar{i}-\mbar{i}^{DP}} \leq \frac{2\sqrt{k}\norm{E_{6,G}}}{\ds}$  where $E_{6,G}$ is the Gaussian tensor added to the whitened tensor $\swt$. An identical bound is established for the eigenvalues, i.e. $|\abar{i}-\abar{i}^{DP}|\leq \frac{2\sqrt{k}\norm{E_{6,G}}}{\ds}$. 

Now we can state the utility:
\begin{align}
&\norm{\mu_i -\mu^{DP}_i}  \leq   \ax \sigma_1(\wtshort) \sqrt{\sgm{1}} \frac{2\sqrt{k}\norm{E_{6,G}}}{\ds}  \nn \\ 
&+ \ax \sigma_1(\wtshort) \frac{\sqrt{\sgm{1}}}{\sgm{k}} \norm{E_{8,G}}\nn \\
&+  \ax \sqrt{\sgm{1}+\norm{E_{8,G}}}  \frac{2 \sqrt{k} \norm{E_{6,G}}}{\ds}
\end{align}
We note that w.h.p we have the following bounds on spectral norms of noisy Gaussian matrix and noisy Gaussian tensor. In particular, $\norm{E_{6,G}}=O(\frac{k^2}{N \tilde{\sigma}_k^{3/2}} \tdp{6})$ and $\norm{E_{8,G}} = O(\frac{\sqrt{d}}{N} \tdp{8})$. This leads to the following utility
\begin{equation*}
\norm{\mu_i -\mu^{DP}_i} = O(\frac{\sqrt{\sigma_1(\slda) k^{2.5}}}{\ds N \tilde{\sigma}_k^{3/2}} \tdp{6} + \frac{\sqrt{\sigma_1(\slda) d}}{\sk N} \tdp{8} + \sqrt{\sgm{1} + \frac{\sqrt{d}}{N} \tdp{8}} \frac{k^{2.5} \tdp{6}}{\ds N \tilde{\sigma}_k^{3/2}}).
\end{equation*}

\subsection{Perturbation on the output of tensor decomposition $\bar{\mu_i}$,$\bar{\alpha_i}$ and $\slda$ Config. 3 ($e_7,e_8$): Proof for Theorem~\ref{thm:utility-config3}}\label{sec:appendix-utility-config3}
This configuration shares edge 8 with the previous. This enables us to borrow the same bounds for the pseudo-inverse $W^\dagger$. Specifically, we have:
\begin{align*}
\norm{W^\dagger-(W^{DP})^\dagger} & \leq \frac{\sqrt{\sgm{1}} \norm{E_{8,G}}}{\sgm{k}}\\
\norm{(W^{DP})^\dagger} & \leq \sqrt{\sgm{1}+\norm{E_{8,G}}}
\end{align*}
In this method, noise is added directly to the eigenvectors and eigenvalues resulting from the tensor decomposition. Therefore, we have: 
\begin{align*}
\mbar{i}^{DP} &= \mbar{i} + Y,   & Y \sim  \mathcal{N}(0,\Delta^2_{\epsilon,\delta} I_k)\\
\abar{i}^{DP} &= \abar{i} + n_i, & n_i \sim  \mathcal{N}(0,\Delta^2_{\epsilon,\delta})
\end{align*}
where $\Delta_{\epsilon,\delta} = \frac{\sqrt{2k} \swt}{\ds} \tdp{7}$ with $\tdp{7} = \frac{\sqrt{2ln(1.25/\delta_7)}}{\epsilon_7}$. This leads to the following bound:
\begin{align*}
\norm{\mu_i-\mu_i^{DP}} \leq  & \ax \sigma_1(\wtshort) \sqrt{\sgm{1}} \norm{Y} + 
\\ &\ax \sigma_1(\wtshort) \frac{\sqrt{\sgm{1}}}{\sgm{k}} \norm{E_{8,G}} +  \\ 
&\ax \sqrt{\sgm{1}+\norm{E_{8,G}}} |n_i|
\end{align*} 
As before w.h.p $\norm{E_{6,G}}=O(\frac{\sqrt{d}}{N} \tdp{6})$.The following bounds hold on $\norm{Y}$ and $|n_i|$, because they are a Gaussian vector and variable. 
In particular, w.h.p. $\norm{Y}=O(\frac{k^{5/2}}{N \tilde{\sigma}_k^{3/2} \tilde{\ds}}\tdp{7})$ and $|n_i|=O(\frac{k^2}{N \tilde{\sigma}_k^{3/2} \tilde{\ds}}\tdp{7})$. 
This leads to the following utility: 
$O(\frac{\sqrt{\sigma_1(\slda) k^{2.5}}}{\tilde{\ds} N \tilde{\sigma}_k^{3/2}} \tdp{7} + \frac{\sqrt{\sigma_1(\slda) d}}{\sk N} \tdp{8} + \sqrt{\sgm{1} + \frac{\sqrt{d}}{N} \tdp{8}} \frac{k^{2} \tdp{7}}{\tilde{\ds} N \tilde{\sigma}_k^{3/2}})$. 


\subsection{Perturbation on the final output $\mu_i$, $\alpha_i$ Config. 4 ($e9$): Proof for Theorem~\ref{thm:utility-config4}}\label{sec:appendix-utility-config4}
In this configuration, we add noise proportional to the output's sensitive
\begin{equation*}
\mu_{i}^{DP} = \mu_{i} + Z, \text{ where }Z \sim  \mathcal{N}(0,\Delta^2_{\epsilon,\delta} I_k)\end{equation*}
where $\Delta_{\epsilon,\delta} = \sv \tdp{9}$, with $\tdp{9} = \frac{\sqrt{2ln(1.25/\delta_9)}}{\epsilon_9}$. 
Similar to the previous analysis, since $Z$ is Gaussian, then w.h.p. $\norm{Z}=O(\frac{\sqrt{d\sgm{1}} k^2}{N \tilde{\ds} \tilde{\sigma}_k^{3/2}})$. We have the utility $O(\frac{\sqrt{\sgm{1}d}k^2}{N \tilde{\ds}  \tilde{\sigma}_k^{3/2}} \tdp{9})$.


\section{K. Evaluation}\label{sec:eval_app}
Recall the grid search process used to choose the best configuration for a particular $\epsilon$. For a particular edge set, we vary the $\epsilon$ parameter across each edge (making sure they still sum to the same composite $\epsilon$) and choosing the combination that yields the maximum likelihood. We do this across every possible configuration, and choose the configuration which yields the highest likelihood to the original output. However, this approach leaks the data likelihood. In order to make the max likelihood computation differentially private, we perturb the sufficient statistics of the data. The sufficient statistic necessarily contains all information to compute the likelihood estimate, so having a differentially private sufficient statistic would guarantee a differentially private likelihood computation. 

Consider the sufficient statistics matrix $S$, which is $K\times d$ (for k topics and a vocabulary of size d). Then, for each entry $s_k^v$, the sensitivity $\triangle_{s_k^v}$ is bounded above by $\frac{d}{N}$. Then, using the gaussian mechanism, we can add noise to each entry: $\overline{s}_k^v$ = $s_k^v + G(\triangle_{s_k^v}, \epsilon, \delta)$, where $G$ denotes the gaussian mechanism. Then $S$ is $(\epsilon, \delta)$ differentially private, and can be used to compute differentially private likelihoods.


\section{L. Experiments on Wikipedia Dataset}\label{sec:wiki-result}

\begin{table}[H]
\small{
\begin{tabular}{c|l}
Topic \# & Top Words\\
\hline
1 & 'air', 'th', 'force', 'squadron', 'special', 'operations', 'test', 'august', 'december', 'mission' \\
\hline
2 & 'music', 'composition', 'kaufmann', 'works', 'dieter', 'vienna', 'acoustic', 'electro', 'des', 'president' \\
\hline
3 & 'river', 'barnstaple', 'yeo', 'taw', 'flooding', 'tributary', 'rivers', 'south', 'devon', 'flows'
\end{tabular}
}
\caption{Top words in Wikipedia Dataset recovered from our \textsf{DP} algorithm with $\epsilon = 1$. Note that the topics shown here are not cherry picked but randomly selected.}
\end{table}
%
%
%
%
%
%


\section{M. Some Useful Identities and Theorems}\label{sec:thms}
\begin{identity} [Square of Sum] \label{I1}
\begin{align*}
\Big( \sum_{i=1}^{N} a_i \Big)^2 = \sum_{i=1}^{N} a_i^2 + \sum_{\substack{i=1,j=1 \\ i\neq j}}^{N,N} a_i a_j 
\end{align*}
\end{identity}

\begin{identity} [Cube of Sum] \label{I2}
\begin{align*}
\Big( \sum_{i=1}^{N} a_i \Big)^3 = \sum_{i=1}^{N} a_i^3 + 3\sum_{\substack{i=1,j=1 \\ i\neq j}}^{N,N} a_i^2 a_j + \sum_{\substack{i=1,j=1,k=1 \\ i \neq j,j \neq k,k \neq i}}^{N,N,N} a_j a_j a_k 
\end{align*}
\end{identity}

\begin{theorem} [Weyl's theorem; Theorem 4.11, p. 204 in \cite{Stewart90matrixperturbation}]. Let $A, E$ be given $m \times n$ matrices with $m \geq n$, then
\begin{align*}
\max_{i \in [n]} |\sigma_i(A)-\sigma_i(A+E)| \leq \norm{E}_2
\end{align*}
\end{theorem}

\begin{theorem}[Wedin's theorem; Theorem 4.11, p. 204 in \cite{Stewart90matrixperturbation}] 
Let $A, E \in \Rbb^{m,n}$ with $m \ge n$ and $\hat{A}=A+E$, the following be the singular value decomposition of $A$
$ \begin{bmatrix}
        U^T_1 \\   
        U^T_2 \\
        U^T_3 
\end{bmatrix}     A \begin{bmatrix}
        V_1 & V_2
\end{bmatrix}. = \begin{bmatrix}
        \Sigma_1 & 0\\   
        0 & \Sigma_2\\
        0 & 0  
\end{bmatrix}.$ 
 Let $\hat{A}$ have a similar decomposition, with $(\hat{U_1},\hat{U_2},\hat{U_3}, \hat{\Sigma_1},\hat{\Sigma_2},\hat{V_1},\hat{V_2})$. And let $\Phi$ is the matrix of canonical angels between range($U_1$) and range($\hat{U_1}$) and $\Theta$ is the matrix of canonical angels between range($V_1$) and range($\hat{V_1}$). If there exists an $\delta, \alpha >0$ such that $min_i \sigma_i(\hat{\Sigma_1}) \geq \alpha + \delta $ and $max_i \sigma_i(\Sigma_2)$, then
\begin{align*} 
max({\norm{\Phi}_2,\norm{\Theta}_2}) \leq \frac{\norm{E}_2}{\delta}.
\end{align*}
\end{theorem}

\begin{theorem}[Bound on the norm of a Gaussian Random Variable] 
Let $n$ be a Gaussian $\mathcal{N}(0,\sigma)$. Then $\Pbb[|n|\leq t] \ge 1-2e^{\frac{-t^2}{2\sigma^2}}$
\end{theorem}

\begin{theorem}[Bound on the norm of a Gaussian Vector] 
Let $Y \sim \mathcal{N}(0,\sigma I_k)$, then $\Pbb[\norm{Y}_2^2 \ge \sigma^2(k+2\sqrt{kt}+2t) ] \le e^{-t}$.
\end{theorem}
\begin{proof}
The proof is immediate from Theorem 2.1 in \cite{hsu2012tail} with $A=I, \mu=0$. 
\end{proof}

\begin{theorem}[Bound on the spectral norm of a Gaussian Matrix \cite{tao2012topics}]\label{gmb}
Let $E \in \Rbb^{d\times d}$ be a symmetric Gaussian matrix with elements sampled iid from $\mathcal{N}(0,\sigma)$, then $\Pbb[\norm{E}_2 = O(\sqrt{d} \sigma)] \ge 1-negl(d)$. 
\end{theorem}

\begin{theorem}[Bound on the spectral norm of a Gaussian Tensor \cite{tomioka2014spectral}]\label{gtb}
Let $E$ be a $K^{th}$ order tensor with each $E_{i_1,\dots,i_K}$ be sampled i.i.d. from a Gaussian $\mathcal{N}(0,\sigma)$, then $\Pbb[\norm{E}_2 \leq \sqrt{8 \sigma^2 (\sum_{i=1}^{K} d_i)\ln(2K/K_0) + \ln(2/\delta) }] \ge 1-\delta $, where $K_0=\ln(3/2)$. Note by extension the bound also holds if the tensor is symmetric as well. 
\end{theorem}

\end{document}